\documentclass[a4paper,fleqn]{cas-sc}

\usepackage[authoryear,longnamesfirst]{natbib}
\usepackage{graphicx} 
\usepackage{bm}
\def\d{\mbox{d}}
\def\E{\mbox{E}}
\usepackage{placeins}
\usepackage{float}
\usepackage{amsfonts, amsmath, amssymb, amsthm}
\newtheorem{assumption}{Assumption}[section]
\newtheorem{theorem}{Theorem}[section]
\usepackage{hyperref}
\newtheorem{example}{Example}
\newtheorem{corollary}{Corollary}[section]
\usepackage{mathtools}
\usepackage{booktabs}
\usepackage{stfloats}

\newtheorem{definition}{Definition}%
\def\tsc#1{\csdef{#1}{\textsc{\lowercase{#1}}\xspace}}
\tsc{WGM}
\tsc{QE}

\begin{document}
\let\WriteBookmarks\relax
\def\floatpagepagefraction{1}
\def\textpagefraction{.001}

\shorttitle{Hierarchical RBF-KAN and RBF-SKAN Architectures}    

\shortauthors{M. Xia and Q. Shen}  

\title [mode = title]{Hierarchical RBF-KAN and RBF-SKAN Architectures for Multidimensional Function Approximation and Random Field Learning}  



%

\author[1,2]{Mingtao Xia}[orcid=0000-0002-2116-4712]

\cormark[1]


\ead{mxia4@uh.edu;xiamingtao97@g.ucla.edu}

\ead[url]{https://sites.google.com/nyu.edu/mingtao-xia/home}

\credit{Writing – review \& editing, Writing – original draft, Visualization, Validation, Supervision, Software, Project administration, Methodology, Investigation, Formal analysis, Conceptualization}

\affiliation[1]{organization={University of Houston},
            addressline={Philip Guthrie Hoffman Hall, 3551 Cullen Blvd}, 
            city={Houston},
            postcode={77204}, 
            state={Texas},
            country={United States of America}}

\affiliation[2]{organization={University of Birmingham},
            addressline={Watson Building}, 
            city={Birmingham},
            postcode={B15 2TT}, 
            state={},
            country={United Kingdom}}

\author[3]{Qijing Shen}[orcid=0009-0009-6685-0861]

\fnmark[2]

\ead{qijing.shen@ndm.ox.ac.uk}

\ead[url]{https://www.chg.ox.ac.uk/people/qijing-shen}

\credit{Writing – review \& editing, Visualization, Software, Methodology, Investigation}

\affiliation[3]{organization={University of Oxford},
            addressline={Henry Wellcome Building for Molecular Physiology, Old Road}, 
            city={Oxford},
            postcode={ OX3 7BN}, 
            state={Oxfordshire},
            country={United Kingdom}}

\cortext[1]{Corresponding author}



\begin{abstract}
In this manuscript, we propose and analyze hierarchical Kolmogorov--Arnold neural network architectures employing radial basis functions as activation functions for approximating deterministic functions and random field models. Specifically, we develop a hierarchical radial-basis-function Kolmogorov--Arnold network (hierarchical RBF-KAN) for multidimensional deterministic function approximation and a hierarchical radial-basis-function stochastic Kolmogorov--Arnold network (hierarchical RBF-SKAN) for random field learning. 
From a theoretical perspective, we establish universal approximation results for both architectures. In particular, we derive quantitative approximation estimates for the hierarchical RBF-KAN, showing that the proposed framework has the potential to partially alleviate the curse of dimensionality in learning high-dimensional functions by reducing the effective dimensionality of the approximation problem. Furthermore, we show that the hierarchical RBF-SKAN can approximate random field models under the Wasserstein-2 metric. Empirically, we show that our proposed radial-basis-function-based neural network structure could effectively learn multivariate functions and random field models. 
\end{abstract}


\begin{highlights}
\item We propose hierarchical RBF-KAN and hierarchical RBF-SKAN architectures for learning multidimensional functions and random fields, respectively. The proposed hierarchical layer design improves approximation accuracy over existing RBF-KAN and multilayer RBF neural network architectures, and our hierarchical RBF-SKAN framework outperforms prevailing uncertainty quantification approaches such as CNF and CVAE models.
\item We establish universal approximation results for the proposed architectures, showing that the hierarchical RBF-KAN can partially alleviate the curse of dimensionality in multivariate function approximation, while the hierarchical RBF-SKAN possesses universal approximation capability for random field learning.
\item We validate the proposed methods through extensive numerical experiments on multidimensional function approximation, chaotic dynamical system reconstruction, and random field learning tasks. We also demonstrate that incorporating ResNet techniques into multi-block hierarchical RBF-KAN architectures significantly improves approximation accuracy and training performance.
\end{highlights}


\maketitle

\section{Introduction}

Radial basis function neural networks (RBFNNs) are a class of feedforward neural networks that employ radial basis functions (RBFs) as activation units in the hidden layer, typically combined with a linear output layer. An RBFNN approximates a target function through a weighted superposition of RBFs centered at selected points in the input space, enabling efficient representations of multivariate nonlinear mappings with strong approximation guarantees. Owing to their localized activation structure, RBFNNs often exhibit favorable optimization and generalization properties, including faster convergence, reduced sensitivity to initialization, and improved interpretability compared with fully connected multilayer perceptrons. Since the seminal works establishing their theoretical foundations and practical formulations \citep{BroomheadLowe1988, Powell1987, MoodyDarken1989, poggio1990networks}, RBFNNs have been widely applied to function approximation, scattered data interpolation, system identification, time-series prediction, classification, and the numerical solution of partial differential equations \citep{Buhmann2003, Fasshauer2007}.

From a theoretical perspective, single-layer RBFNNs possess the universal approximation property; namely, they can approximate arbitrary continuous functions on compact domains to any prescribed accuracy under suitable conditions \citep{park1991universal, poggio1990networks, wu2012using}. These results provide the theoretical justification for employing RBF-based neural networks to learn unknown nonlinear mappings. Beyond shallow architectures, increasing attention has recently been devoted to multilayer neural networks equipped with RBF activations \citep{chao2001multilayer, zhao2019multi, jiang2022efficient}, motivated by the empirical success of deep architectures in high-dimensional learning tasks. In parallel, RBF representations have also been incorporated into emerging new neural architectures \cite{chao2026rbf}, including Kolmogorov--Arnold networks and related operator-learning frameworks \citep{liu2024kan}.
Despite the strong empirical performance of RBF-based neural architectures across a broad range of applications, comparatively limited theoretical work has addressed their approximation and expressivity properties for learning complex multivariate functions. In particular, rigorous analyses characterizing how depth and network width influence approximation efficiency and representation power remain relatively scarce. 

In this work, we propose and analyze a hierarchical RBF-based Kolmogorov--Arnold network (hierarchical RBF-KAN) architecture and a hierarchical RBF-based stochastic Kolmogorov--Arnold network (hierarchical RBF-SKAN) for effectively learning multidimensional deterministic functions and random fields, respectively. Unlike previous multilayer RBF neural networks \citep{jiang2022efficient}, our proposed hierarchical RBF-KAN employs a hierarchical architecture rather than standard fully connected layers. Moreover, the proposed framework extends recent RBF-KAN models \citep{li2024kolmogorov,chao2026rbf} by allowing distinct numbers of neurons across different layers with a hierarchical architecture, aiming to more faithfully reproduce the Kolmogorov--Arnold representation structure. As demonstrated both theoretically and empirically, this hierarchical architecture is essential for effectively learning multivariate functions. From a theoretical perspective, we establish quantitative approximation results that relate the number of neurons required to achieve a prescribed approximation accuracy. These results provide insight into how the proposed hierarchical RBF neural network architecture can partially alleviate the curse of dimensionality as the dimensionality of the input increases. 


The main contributions of this manuscript are summarized as follows:
\begin{itemize}
\item We propose hierarchical RBF-KAN and hierarchical RBF-SKAN architectures for efficiently learning multidimensional functions and random fields. Specifically, the proposed RBF-KAN framework generalizes existing RBF-KAN models by introducing hierarchical structures across different network layers, inspired by the Kolmogorov--Arnold representation theorem. This hierarchical design substantially improves the approximation accuracy for multivariate functions compared with existing RBF-KAN and multilayer RBF neural network architectures.
For random field reconstruction, the proposed RBF-SKAN framework demonstrates superior performance relative to several prevailing machine-learning-based uncertainty quantification approaches, including the conditional normalizing flow (CNF) and conditional variational autoencoder (CVAE) frameworks.
\item From a theoretical perspective, we establish approximation results for the proposed RBF neural network architectures. In particular, we prove a universal approximation theorem showing that the proposed hierarchical RBF-KAN can partially alleviate the curse of dimensionality
in approximating multivariate functions. We further establish the universal approximation capability of the proposed hierarchical RBF-SKAN architecture for learning random field models.
\item We demonstrate the effectiveness of the proposed RBF neural network architectures through a range of numerical experiments, including multidimensional function approximation, dynamical system reconstruction, and random field learning. Also, we show that incorporating residual-network (ResNet) techniques into multi-block hierarchical RBF-KAN can enhance approximation accuracy and training performance.
\end{itemize}

The remainder of this manuscript is organized as follows.
In Section~\ref{section2}, we introduce the proposed hierarchical RBF-KAN and hierarchical RBF-SKAN architectures for learning deterministic functions and random field models, respectively. Specifically, we analyze their universal approximation properties and discuss how our proposed hierarchical RBF-KAN framework can partially alleviate the ``curse of dimensionality'' in approximating multidimensional functions.
In Section~\ref{section3}, we present a series of numerical experiments to demonstrate the effectiveness of the proposed hierarchical RBF-KAN and hierarchical RBF-SKAN and compare their performance with several existing RBF-based neural networks.
Finally, in Section~\ref{section4}, we summarize the main findings and outline potential directions for future research.



\section{Universal approximation ability of our proposed hierarchical RBF-KAN and hierarchical RBF-SKAN}
\label{section2}

In this section, we introduce the proposed hierarchical RBF-KAN and RBF-SKAN architectures and analyze their approximation properties. In particular, we establish universal approximation results for efficiently representing both multivariate deterministic functions and random fields using the proposed RBF-KAN and RBF-SKAN frameworks. 
Specifically, for deterministic function approximation, unlike recent KAN results that primarily establish asymptotic universal approximation properties \citep{chiu2026free}, we derive quantitative approximation estimates that explicitly relate the approximation error to the number of neurons in each layer. These results show that, for a broad class of multivariate functions, there exists a hierarchical RBF-KAN architecture capable of reducing the effective dimensionality of the approximation problem, thereby partially alleviating the curse of dimensionality.


\subsection{Universal approximation theorem of the hierarchical RBF-KAN to learn multivariate deterministic functions}
\label{subsec2:1}
First, we consider the problem of learning a multivariate function 
\begin{equation}
\bm{y}(\bm{x}) = u(\bm{x}), 
\qquad \bm{x} \in \Omega \subseteq \mathbb{R}^d.
\label{deterministic_model}
\end{equation}
We propose a hierarchical RBF-KAN, illustrated in Fig.~\ref{fig:rbfnn}, to approximate Eq.~\eqref{deterministic_model}. One primary distinction between the proposed hierarchical RBF-KAN and existing RBF-KAN architectures, such as the model in \cite{chao2026rbf}, is that the proposed hierarchical RBF-KAN employs hierarchical ``blocks'' consisting of two distinct activation layers, rather than fully connected hidden layers with a uniform number of neurons throughout the network. Within one block, the number of neurons in the first post-activation layer (\(n_2\)) is chosen to be \((2d+1)\) times the number of neurons in the second post-activation layer, thereby more faithfully reproducing the Kolmogorov--Arnold representation described in Eq.~\eqref{KAN_expression}.
 Throughout this work, we use the following Gaussian kernel as the RBF activation function for each neuron:
\begin{equation}
B(x) \coloneqq \exp\!\left(-x^2\right).
\label{Gaussian_kernel}
\end{equation}

    \begin{figure}
    \centering
\includegraphics[width=0.9\linewidth]{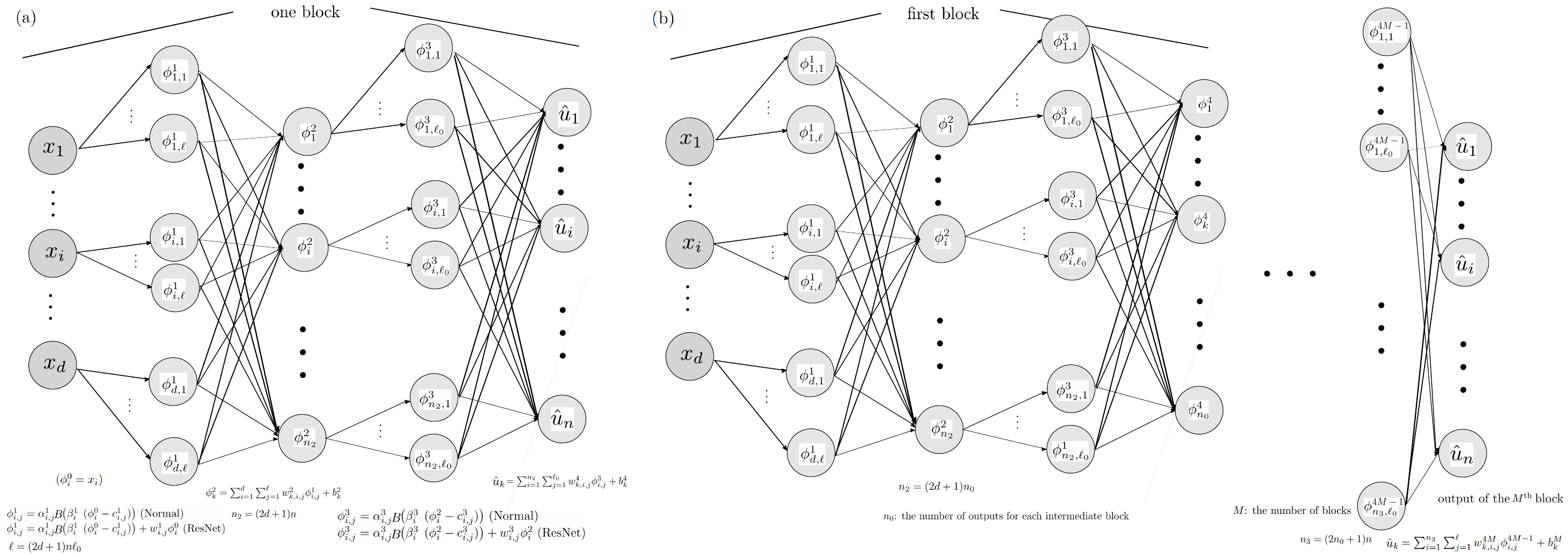}
    \caption{
    The structure of the proposed hierarchical RBF-KAN is illustrated in Fig.~\ref{fig:rbfnn}. Panel (a) shows a single-block hierarchical RBF-KAN, while panel (b) presents a multi-block hierarchical RBF-KAN architecture. Note that, in the illustration of the single-block hierarchical RBF-KAN shown in panel~(a), each $\phi_{i,j}^1$ depends only on the input variable $x_i$, while each $\phi_{i,j}^3$ depends only on $\phi_i^2$. Consequently, these two activation layers are not fully connected (dense) layers. 
    Each block of the hierarchical RBF-KAN consists of two activation layers, where the activation function $B$ is the Gaussian kernel defined in Eq.~\eqref{Gaussian_kernel}. Before applying the Gaussian activation function, each input variable is replicated $\ell$ times. 
In panel (a), the quantity $\phi_{\lfloor (i-1)/d \rfloor}^2$ is designed to approximate the input of 
$\Phi_{i-\lfloor (i-1)/d \rfloor(2d+1)}$,
namely, $\sum_{q=1}^d
\phi_{i-\lfloor (i-1)/d \rfloor(2d+1),q}(x_q)$,
in the Kolmogorov--Arnold representation given in Eq.~\eqref{KAN_expression} for the $\lfloor (i-1)/d \rfloor^{\text{th}}$ component of the function to be learned. The network may employ either the standard feedforward propagation strategy or a ResNet structure \citep{he2016deep} for forward propagation.}
    \label{fig:rbfnn}
\end{figure}

In Fig.~\ref{fig:rbfnn}, both the scales $ \beta_{i}^k $ and centers $ c_{i,j}^k $ for the RBFs are trainable parameters. Forward propagation may follow either a standard feedforward architecture or a ResNet-type structure.

Next, we shall analyze the approximation ability of our proposed hierarchical RBF-KAN in Fig.~\ref{fig:rbfnn}. For simplicity, we assume:
\begin{equation}
\bm{x}\in \Omega = [-1,1]^d.
\end{equation}
Following \cite{barthelmann2000high}, we define the function space
\begin{equation}
F_k^d 
=
\left\{
u : [-1,1]^d \to \mathbb{R}
\;\middle|\;
D^{\boldsymbol{\alpha}} u \text{ is continuous for all } 
\boldsymbol{\alpha} = (\alpha_1,\dots,\alpha_d) \in \mathbb{N}_0^d
\text{ with } \alpha_i \le k
\right\},
\end{equation}
equipped with the norm
\begin{equation}
\|u\|_{k,\infty}
\coloneqq
\max_{\substack{\boldsymbol{\alpha} \in \mathbb{N}_0^d, \alpha_i \le k}}
\| D^{\boldsymbol{\alpha}} u \|_{\infty},
\end{equation}
which contains functions whose mixed partial derivatives up to order $k$ in each coordinate are continuous and bounded.

For the continuous multivariate function in Eq.~\eqref{deterministic_model}
\begin{equation}
u(\bm{x}) = u(x_1,\dots,x_d),
\end{equation}
the Kolmogorov--Arnold representation theorem states that there exist continuous univariate functions
\begin{equation}
\phi_{q,p} : [-1,1] \to \mathbb{R},
\qquad
\Phi_q : \mathbb{R} \to \mathbb{R},
\end{equation}
such that
\begin{equation}
u(x_1,\dots,x_d)
=
\sum_{q=0}^{2d}
\Phi_q
\!\left(
\sum_{p=1}^{d}
\phi_{q,p}(x_p)
\right).
\label{KAN_expression}
\end{equation}
This representation reduces the approximation of a $d$-dimensional function to the approximation of finitely many univariate functions. Therefore, the multivariate approximation problem can be decomposed into the approximation of univariate functions $ \phi_{q,p} $ and $ \Phi_q $.
Since the domain $[-1,1]$ is compact and all functions involved are continuous, the univariate functions $ \phi_{q,p} $ and $ \Phi_q $ are uniformly continuous. Hence, for any $\varepsilon > 0$, there exist moduli of continuity $\delta_{\phi}(\varepsilon)$ and $\delta_{\Phi}(\varepsilon)$ such that
\begin{equation}
|\phi_{q,p}(x) - \phi_{q,p}(x+\varepsilon)|
\le
\delta_{\phi}(\varepsilon), \,\, \forall q, p
\label{uniform1}
\end{equation}
and
\begin{equation}
|\Phi_q(z) - \Phi_q(z+\varepsilon)|
\le
\delta_{\Phi}(\varepsilon), \,\, \forall q
\label{uniform2}
\end{equation}
with
\begin{equation}
\delta_{\phi_{q, p}}(\varepsilon) \to 0,
\qquad
\delta_{\Phi_q}(\varepsilon) \to 0,
\qquad 
\text{uniformly for all } q, p  ~\text{as}~ \varepsilon \to 0.
\label{uniform3}
\end{equation}
Note that, according to the Kolmogorov--Arnold representation theorem, the inner functions \(\phi_{q,p}\) are universal, whereas the outer functions \(\Phi_q\) depend on the target function \(f\). Consequently, the corresponding modulus of continuity \(\delta_{\Phi_q}\) may also depend on the specific choice of \(f\).
Moreover, the Kolmogorov--Arnold representation in Eq.~\eqref{KAN_expression} implies that an accurate approximation of the univariate component functions \(\phi_{q,p}\) and \(\Phi_q\) leads to an accurate approximation of the full multivariate function \(u\) in Eq.~\eqref{deterministic_model}. In particular, adaptive learning of the RBF centers and scale parameters can significantly improve approximation efficiency \citep{billings2007generalized}. Consequently, it suffices to show that the proposed hierarchical RBF-KAN is capable of approximating arbitrary continuous univariate functions. In the following, we establish a universal approximation result for the proposed hierarchical RBF-KAN and further illustrate how the resulting approximation framework can partially alleviate the curse of dimensionality in multivariate function learning under suitable conditions.

\begin{theorem}
\rm
\label{theorem1}
Let \(u(\bm{x})\), with \(\bm{x}\in\Omega=[-1,1]^d\), be a scalar-valued function that admits the Kolmogorov--Arnold representation in Eq.~\eqref{KAN_expression}. 
 Then, for any $c>0$, there exists a hierarchical RBF-KAN of the form shown in Fig.~\ref{fig:rbfnn}, whose output is denoted by $\hat{u}$, such that
\begin{equation}
\begin{aligned}
    \|u-\hat{u}\|_{\infty}\leq c.
\end{aligned}
\label{thm1_statement}
\end{equation}
\end{theorem}

The proof of Theorem~\ref{theorem1} is provided in Appendix~\ref{proof_theorem1}. Note that Theorem~\ref{theorem1} does not require \(\Phi_q\) or \(\phi_{q,p}\) to be smooth; only uniform continuity is assumed, which is automatically satisfied when the domain of \(f\) is compact. Hence, Theorem~\ref{theorem1} applies to a broad class of continuous functions and justifies the use of our hierarchical RBF-KAN for the general approximation of functions. 

To illustrate how the proposed hierarchical RBF-KAN may alleviate the curse of dimensionality, we further assume that $\Phi_q$ and $\phi_{q,p}$, together with their first-order derivatives, are uniformly bounded by a constant $M$. Taking $N(\varepsilon)\coloneqq \varepsilon^{-10}$, the hierarchical RBF-KAN approximation error bound Eq.~\eqref{thm1_bound} can then be simplified as
\begin{equation}
\begin{aligned}
    &\left|u(\bm{x})-\sum_{q=0}^{2d}\Phi_{q,N}\!\left(\sum_{p=1}^d \phi_{q,p,N}(x_p)\right)\right| \\
    &\leq
    \sum_{q=0}^{2d}\sum_{p=1}^d M\,\left|\phi_{q,p}(x_p)-\phi_{q,p,N}(x_p)\right|
    +
    \sum_{q=0}^{2d}
    \Bigg(
    4M\bigl(1-\Psi(\varepsilon^{-1})\bigr)
    +2M\varepsilon \\
    &\qquad\qquad
    +\tilde c_{d,k}\bigl(4(a_q+2)\bigr)^k
    \varepsilon^{-k}\log N\,
    \|B_{1}\|_{\infty,k}\|\Phi_q\|_{\infty} 
    +2\bigl(1-\Psi(\varepsilon^{-1})\bigr)(a_q+1)\|\Phi_q\|_{\infty}
    \Bigg)\\
    &\leq \sum_{q=0}^{2d}\sum_{p=1}^{d}M\Big(4M\bigl(1-\Psi(\varepsilon^{-1})\bigr)
    +2M\varepsilon \\
    &\qquad\qquad
    +\tilde c_{d,k}\bigl(4(x_{q,p}+1)\bigr)^k
    \varepsilon^{-k}(-10)\log \varepsilon\,
    M\,\|B_{1}\|_{\infty,k} 
    +2\bigl(1-\Psi(\varepsilon^{-1})\bigr)x_{q,p}M\Big)
    \\
    &\quad+\sum_{q=0}^{2d}
    \Bigg(
    4M\bigl(1-\Psi(\varepsilon^{-1})
    +2M\varepsilon 
    \\
    &\qquad\qquad+\tilde c_{d,k}\bigl(4(a_q+2)\bigr)^k
    \varepsilon^{-k}(-10)\log \varepsilon\,
    M\|B_{1}\|_{\infty,k}
    +2\bigl(1-\Psi(\varepsilon^{-1})\bigr)(a_q+1)M
    \Bigg),
\end{aligned}
\label{further_simplify}
\end{equation}
where $\Psi(\varepsilon^{-1})
\coloneqq
\int_{-\varepsilon^{-1}}^{\varepsilon^{-1}}
\frac{1}{\sqrt{2\pi}}
\exp\!\left(-\frac{x^2}{2}\right)\d x$,
\(x_{q,p}\) and \(a_q\) are parameters associated with the Kolmogorov--Arnold representation in Eq.~\eqref{KAN_expression}, and \(\tilde{c}_{d,k}\) denotes a positive constant. In Eq.~\eqref{further_simplify}, the right-hand side is of order at most \(\mathcal{O}(\varepsilon)=\mathcal{O}(N^{-1/10})\) for \(k\geq 2\), and the resulting convergence rate as $\varepsilon=N^{-\frac{1}{10}}\rightarrow0$ does not depend explicitly on the dimension of the input variable $\bm{x}$. Specifically, $\hat{u}=\sum_{q=0}^{2d}\Phi_{q,N}\big(\sum_{p=1}^d\phi_{q,p,N}(x_p)\big)$ is the output of the single-block hierarchical RBF-KAN without the ResNet technique, as shown in Fig.~\ref{fig:rbfnn}(a). Therefore, under mild conditions on $u$ and on the first-order derivatives of $\phi_{q,p,N}$ and $\Phi_{q,N}$, the approximation accuracy of the optimal single-block RBF-KAN does not explicitly deteriorate in the derived convergence estimate as the dimension $d$ increases.

For the multi-block hierarchical RBF-KAN equipped with the ResNet technique, as shown in Fig.~\ref{fig:rbfnn}, the same approximation result also holds, since any function representable by its first block with the ResNet coefficients of the first block being set to zero remains representable by the full network. Finally, we shall demonstrate empirically that using the same network structure as Fig.~\ref{fig:rbfnn} with other activation functions may perform poorly in certain cases, such as the approximation of highly oscillatory functions.

We next extend Theorem~\ref{theorem1} to the approximation of vector-valued functions $\bm{u}(\bm{x})\in\mathbb{R}^n$.

\begin{corollary}
\rm
\label{col1}
Let
\[
\bm{u}(\bm{x})
=
(u_1(\bm{x}),\ldots,u_n(\bm{x})),
\qquad
\bm{x}\in\Omega=[-1,1]^d,
\]
where each continuous component function \(u_i\) admits a Kolmogorov--Arnold representation of the form given in Eq.~\eqref{KAN_expression}.
 Then, for any $c>0$, there exists a hierarchical RBF-KAN of the form shown in Fig.~\ref{fig:rbfnn}, with output denoted by $\hat{\bm{u}}(\bm{x})$, such that
\begin{equation}
\begin{aligned}
    \|\bm{u}-\hat{\bm{u}}\|_{\infty}\leq c,
\end{aligned}
\label{col1_result}
\end{equation}
where
\[
\|\bm{u}\|_{\infty}
\coloneqq
\max_{1\leq i\leq n}\|u_i\|_{\infty}.
\]
\end{corollary}

\begin{proof}
For each component function $u_i$, Theorem~\ref{theorem1} guarantees the existence of a hierarchical RBF-KAN, denoted by $R_i$, whose output $\hat{u}_i$ satisfies the error estimate
\[
\|u_i-\hat{u}_i\|_{\infty}\leq c.
\]
By combining the networks $\{R_i\}_{i=1}^n$ (redundant coefficients in the dense linear layers may be set to zero), we construct a hierarchical RBF-KAN with the output:
\[
\hat{\bm{u}}(\bm{x})
=
(\hat{u}_1(\bm{x}),\ldots,\hat{u}_n(\bm{x})).
\]
Then, by the definition of the vector-valued supremum norm,
\[
\|\bm{u}-\hat{\bm{u}}\|_{\infty}
=
\max_{1\leq i\leq n}
\|u_i-\hat{u}_i\|_{\infty}
\leq c.
\]
This completes the proof.
\end{proof}

\subsection{The universal approximation ability of our hierarchical RBF-SKAN for learning random field models}

In this subsection, we shall develop a stochastic version of the hierarchical RBF-KAN introduced in Subsection~\ref{subsec2:1} (RBF-SKAN). Our RBF-SKAN introduces randomness in the hierarchical RBF-KAN to learn random field models from noisy data. 
The structure of our RBF-SKAN is given in Fig.~\ref{fig:rbfnn}. For each realization, the random variables $w_{i, j}^1$ and the random scales $\beta^2_{i, j, k}, \beta^4{i, j}$ are sampled independently. Therefore, the output of our RBF-SKAN is sampled from a distribution determined by the input $\bm{x}$.
    \begin{figure}
    \centering
\includegraphics[width=0.9\linewidth]{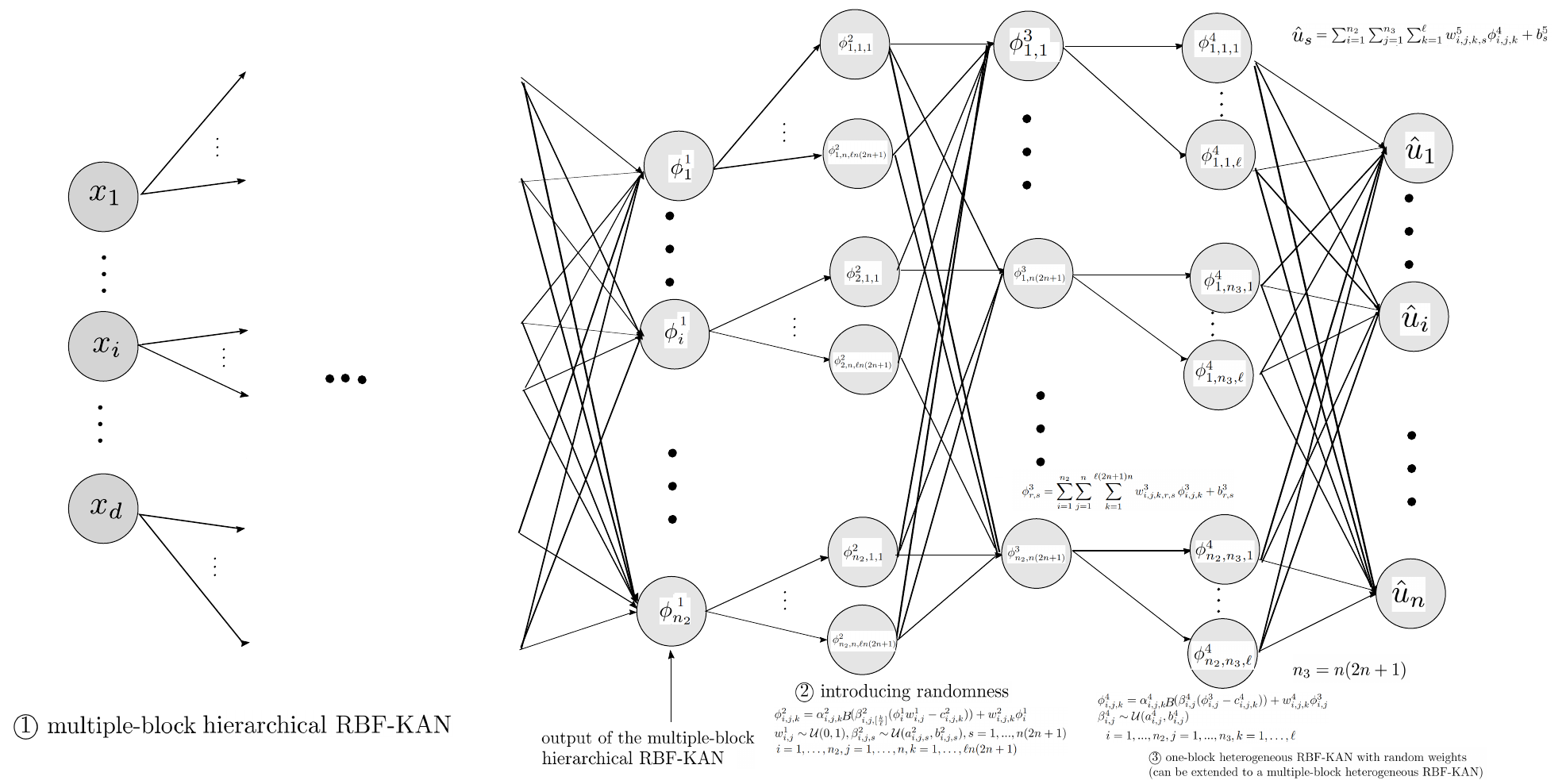}
    \caption{
    The proposed hierarchical RBF-SKAN is constructed by integrating two RBF-KAN modules. The activation function is chosen as the Gaussian kernel Eq.~\eqref{Gaussian_kernel}.
The first module is a multi-block hierarchical RBF-KAN, as shown in Fig.~\ref{fig:rbfnn}, which takes the inputs $x_1,\ldots,x_d$ and outputs $\phi_1^1,\ldots,\phi_{n_2}^1$. The second module, which may also be either one-block or multi-block hierarchical RBF-KAN, takes $\phi_i^{1}w_{i,k}$, $i=1,\ldots,n_2$, $k=1,\ldots,n$, as inputs and produces the output $\hat{\bm{y}}=(\hat{y}_1,\ldots,\hat{y}_n)$. $w_{i, k}$ is sampled from $\mathcal{U}(0, 1)$ independently for each realization to introduce randomness. $B$ is the Gaussian kernel activation function Eq.~\eqref{Gaussian_kernel}. In addition, we introduce randomness into the RBF scale parameters ($\{\beta_{i,j,k}^2\}$ and $\{\beta_{i,j}^4\}$) in the second hierarchical RBF-KAN by letting them be sampled from independent uniform distributions, with both their upper and lower bounds treated as trainable parameters.}
    \label{fig:rbfsnn}
\end{figure}

We shall prove how a special instance of the RBF-SKAN in Fig.~\ref{fig:rbfsnn} could approximate an unknown random field model under the squared Wasserstein-2 ($W_2$) metric:
\begin{equation}
    \bm{y}_{\bm{x}} = \bm{u}(\bm{x};\omega),\qquad \omega\in \Omega
    \label{random_variable}
\end{equation}
under nonrestrictive conditions. For this special case of the RBF-SKAN, we enforce the variances of \(\beta_{i,j,k}^2\) and \(\beta_{i,j}^4\) to be zero; that is, the scale parameters of the RBFs in the second hierarchical RBF-KAN are taken to be deterministic, so that the randomness arises solely from \(w_{i,j}^1\). We then denote the output of this particular hierarchical RBF-SKAN by
\begin{equation}
    \hat{\bm{y}}_{\bm{x}}
    =
    \hat{\bm{u}}(\bm{x};\hat{\omega}),
    \qquad
    \hat{\omega}\in\hat{\Omega}.
    \label{approximate_random_variable}
\end{equation}

Below, we first introduce the $W_2$ distance between the probability distributions associated with $\bm{y}_{\bm{x}}$ and $\hat{\bm{y}}_{\bm{x}}$ in Eqs.~\eqref{random_variable} and \eqref{approximate_random_variable} as:
\begin{definition}
\rm 
\label{def:W2}
For two continuous random variables $\bm{y}, \hat{\bm{y}}\in\mathbb{R}^n$, we assume that
\begin{equation}
    \E[\|\bm{y}\|^2]< \infty,\,\,\,\,\E[\|\hat{\bm{y}}\|^2]< \infty,\,\,\forall \bm{x}\in D
\end{equation}
where $\|\cdot\|$ is a distance metric defined for $\bm{y}$. We denote the probability density functions associated with $\bm{y}$ and $\hat{\bm{y}}$ by $f$ and $\hat{f}$, respectively. The $\bm{W_2}$ \textbf{distance} defined for $f$ and $\hat{f}$ is:
\begin{equation}
W_{2}(f, \hat{f}) \coloneqq \inf_{{\pi_{{f}, \hat f}}}
\E_{(\bm{y}, \hat{\bm
{y}})\sim {{\pi_{{f}, \hat f}}}(\bm{y}, \hat{\bm
{y}})}\big[\|{\bm{y}} - \hat{{\bm{y}}}\|^{2}\big]^{\frac{1}{2}}.
\label{pidef}
\end{equation}
In Eq.~\eqref{pidef}, $\pi_{{f}, \hat f}(\bm{y}, \hat{\bm
{y}})$ is a special coupled measure of the joint random variable $(\bm{y}, \hat{\bm
{y}})$ whose marginals coincide with the probability measures of $\bm{y}$ and $\hat{\bm{y}}$, respectively:
\begin{equation}
\begin{aligned}
\begin{cases}
{\pi_{f, \hat f}}\left(A_1 \times  \mathbb{R}^{d}\right) = \int_{A_1}f(\bm{y})\d\bm{y},\\
{\pi_{{f}, \hat f}}\left( \mathbb{R}^{d}\times A_2\right) = \int_{A_2}\hat f(\hat{\bm{y}})\d\hat{\bm{y}}, 
\end{cases}\hspace{1cm}\forall A_1, A_2\in \mathcal{B}( \mathbb{R}^{d}),
\end{aligned}
\label{pi_def}
\end{equation}
{where $\mathcal{B}(\mathbb{R}^{d})$ denotes the Borel $\sigma$-algebra associated with $\mathbb{R}^{d}$, and the infimum in Eq.~\eqref{pidef} iterates over all coupled distributions $\pi_{f, \hat f}(\bm{y}, \hat{\bm
{y}})$ of $(\bm{y}, \hat{\bm{y}})$ satisfying Eq.~\eqref{pi_def}.}
\end{definition}
We make the following assumptions:
\begin{assumption}
\rm
\label{assumptions_w2}

\begin{enumerate}
    \item $\bm{y}_{\bm{x}}$ and $\hat{\bm{y}}_{\bm{x}}$ in Eqs.~\eqref{random_variable} and \eqref{approximate_random_variable} are uniformly bounded. Without loss of generality, we
    further assume that each component of $\bm{y}_{\bm{x}}$ and $\hat{\bm{y}}_{\bm{x}}$ is in $[0, 1]$.
\item In Eqs.~\eqref{random_variable} and \eqref{approximate_random_variable}, $\omega$ is independent of $\bm{x}$ and $\hat{\omega}$ is independent of $\bm{x}$. 
\item The probability measures associated with random variable $\bm{y}_{\bm{x}}$ in Eq.~\eqref{random_variable} is uniform Lipschitz continuous on $\bm{x}$ 
in the $W_2$ distance sense: 
\begin{equation}
    W_2(f_{\bm{x}}, f_{\tilde{\bm{x}}})\leq L \|\bm{x}-\tilde{\bm{x}}\|_2, \,\,  \forall \bm{x}, \hat{\bm{x}}\in D, 
    \label{l_condition}
\end{equation}
where $f_{\bm{x}}$ is the probability measure associated with $\bm{y}_{\bm{x}}$. To exemplify how Assumptions 3 (and 4 below) can be met, consider the case when $f$ is uniformly Lipschitz continuous in $\bm{x}$ for all $\omega$. 
Then, we have:
\begin{equation}
W_2^2(f_{\bm{x}},f_{\hat{\bm{x}}})\leq \E\big[\|\bm{y}({\bm{x}};\omega)-\bm{y}({\tilde{\bm{x}}};\omega)\|^2\big]\leq L^2 \|\bm{x}-\tilde{\bm{x}}\|_2^2.
\end{equation}
\end{enumerate}
\end{assumption}

We can prove the following asymptotic universal approximation result for the hierarchical RBF-SKAN shown in Fig.~\ref{fig:rbfsnn}.  

\begin{theorem}
    \rm
    \label{theorem2}
    Let $\bm{x}\in\Omega=[-1,1]^d$ be equipped with an associated probability measure $\nu(\cdot)$. Let $f_{\bm{x}}$ denote the probability density function of $\bm{y}(\bm{x};\omega)$ in Eq.~\eqref{random_variable}. Then, for any $c>0$, under Assumption~\ref{assumptions_w2}, there exists a hierarchical RBF-SKAN with deterministic weights, scales, and centers , which can be viewed as a special case of the hierarchical RBF-SKAN shown in Fig.~\ref{fig:rbfsnn} (so randomness is introduced from stochastic $w_{i, j}^1$ in Fig.~\ref{fig:rbfsnn}), such that, if $\hat{\bm{y}}(\bm{x};\hat{\omega})$ denotes its output with associated probability density $\hat{f}_{\bm{x}}$, we have
\begin{equation}
    \int_{\Omega} W_2^2(f_{\bm{x}}, \hat{f}_{\bm{x}})\,\nu(\d \bm{x}) < c,
\end{equation}
where \(\nu(\cdot)\) is a non-degenerate probability measure on \(\Omega\) associated with \(\bm{x}\), and \(\nu(A_n)\to 0\) for any sequence of measurable sets \(A_n\downarrow \varnothing\).
\end{theorem}

The proof of Theorem~\ref{theorem2} is given in Appendix~\ref{proof_theorem2}. Theorem~\ref{theorem2} is an asymptotic result, and obtaining a more quantitative estimate of the number of neurons (or network width) required for the RBF-SKAN to achieve a prescribed approximation accuracy under the $W_2$ metric generally requires additional information about the measure $\nu(\cdot)$ and is therefore problem-dependent.

\section{Numerical examples}
\label{section3}

In this section, we carry out several numerical experiments to evaluate the performance of our proposed hierarchical RBF-KAN architecture shown in Fig.~\ref{fig:rbfnn} and compare it with existing RBF-based and KAN-based neural network architectures for learning multivariate functions and dynamical systems. We also compare the proposed hierarchical RBF-SKAN framework in Fig.~\ref{fig:rbfsnn} with several prevailing machine-learning approaches for uncertainty quantification (UQ) in random field reconstruction problems. All numerical experiments are implemented in Python 3.11 and conducted on a desktop equipped with a 32-core Intel\textsuperscript{\textregistered} Core i9-13900KF CPU. The hyperparameter settings and neural network configurations used in all experiments are summarized in Table~\ref{tab:setting} in Appendix~\ref{training_details}.


\begin{example}
    \rm
    \label{example1}

First, we consider learning a multidimensional extension of the highly oscillatory function studied in \cite{jiang2022efficient,jagtap2020adaptive}:
\begin{equation}
\left\{
\begin{aligned}
z(\bm{x})
&=
x_1
\prod_{i=2}^{d}
\left(
\left|
\frac{10x_i}{9}
\right|^{\frac{i-1}{i}}
\mathrm{sign}(x_i)
\right),
\\
u(\bm{x})
&=
\sin(z(\bm{x}))\cdot
\bigl(z(\bm{x})^3-z(\bm{x})\bigr)
+
\sin\bigl(12z(\bm{x})\bigr),
\end{aligned}
\right.
\label{example1_model}
\end{equation}
where $\bm{x}=(x_1,\ldots,x_d)$ and each $x_i$ is independently sampled from the uniform distribution $\mathcal{U}(-3,3)$. 
We compare the proposed hierarchical RBF-KAN shown in Fig.~\ref{fig:rbfnn} with several existing neural network architectures employing RBF activation functions, including RBF-MLP-I \citep{chao2001multilayer}, RBF-MLP-II \citep{jiang2022efficient}, and the RBF-KAN framework proposed in \cite{chao2026rbf}, as well as multilayer neural networks equipped with alternative activation functions and the spline-based KAN \citep{liu2024kan}.

{\tiny \begin{table}[htbp]
\centering
\caption{Error on the testing set for different neural network structures for learning Eq.~\eqref{example1_model}. For each type of NN structure, the error without equipping the ResNet technique is shown in brackets, and values outside parentheses correspond to errors of ResNet-equipped models. For the RBF-MLP-II, RBF-MLP-II, RELU-MLP, Sigmoid-MLP, and Tanh-MLP neural networks, each hidden layer has the same number of neurons (50). For the RBF-KAN structure, each hidden layer consists of 400 neurons, whose outputs are then fed into a linear layer to generate 50 outputs for the next hidden layer, as detailed in \cite{chao2026rbf}. }
\label{tab:three_blocks}
\tiny
\begin{tabular}{ccccccc}
\toprule
\multicolumn{7}{c}{\textbf{one-block hierarchical RBF-KAN versus two-layer multilayer perceptrons}} \\
\midrule
Dimension of $\bm{x}$  & 1 & 2 & 3 & 4 & 5 & 6 \\
\midrule
\textbf{hierarchical RBF-KAN (ours)} 
& \textbf{2.0912(0.1620)} 
& \textbf{1.5653(0.2803)} 
& \textbf{0.3159(0.3890)} 
& \textbf{0.1767(0.2557)} 
& \textbf{0.0321(0.0814)} 
& \textbf{0.0246(0.2300)} \\
RBF-MLP-I & 0.2660(0.0527) & 0.5544(0.3196) & 0.2748(0.6291) & 0.1972(0.7309) & 0.2631(0.7337) & 0.4052(0.5150) \\
RBF-MLP-II & 0.8209(0.0012) & 0.5152(0.0573) & 0.5637(0.7355) & 0.4811(0.6859) & 0.1015(0.3661) & 0.0700(0.1003) \\
RBF-KAN & 7.9113(0.0283) & 5.3267(0.0289) & 0.4160(0.0147) & 0.2021(0.0317) & 0.2342(0.0252) & 0.0692(0.2343) \\
RELU-MLP & 0.2765(0.1428) & 0.3536(0.1008) & 0.5299(0.3329) & 0.5751(0.4741) & 0.3948(0.7733) & 0.4114(0.7397) \\
Sigmoid-MLP & 0.2817(0.2124) & 0.3423(0.2406) & 0.3245(0.0666) & 0.1368(0.0973) & 0.1316(0.3126) & 0.2162(0.4247) \\
Tanh-MLP & 0.1085(0.0457) & 0.1695(0.0096) & 0.1175(0.0697) & 0.2191(0.1243) & 0.2313(0.2753) & 0.2351(0.3923) \\
spline KAN   &  0.0002(0.0019) &    0.0453(0.0316)  &  0.2950(0.2202)  &  0.2441(0.2109)    & 0.3101(0.2995)  &  0.1676(0.1598)\\
hierarchical Tanh-KAN & 4.0869(0.3316) & 0.2421(0.2817) & 0.3578(0.2531) & 0.3772(0.1345) & 0.0991(0.0718) & 0.0364(0.1291)\\
\midrule
\multicolumn{7}{c}{\textbf{two-block hierarchical RBF-KAN versus four-layer multilayer perceptrons}} \\
\midrule
\textbf{hierarchical RBF-KAN (ours)} 
& \textbf{0.0247(2.0837)} 
& \textbf{0.0335(0.2438)} 
& \textbf{0.0232(0.3829)} 
& \textbf{0.0381(1.0096)} 
& \textbf{0.0525(0.7442)} 
& \textbf{0.0669(0.4156)} \\
RBF-MLP-I & 0.2620(0.0011) & 0.3561(0.3277) & 0.5122(0.6263) & 0.3809(0.6855) & 0.0402(0.7274) & 0.0484(0.3510)\\
RBF-MLP-II & 0.0531(0.0021) & 0.1835(0.0787) & 0.2651(0.4342) & 0.2930(0.7549) & 0.0578(0.4665) & 0.0530(0.4417)\\
RBF-KAN & 2.1974(7.8114) & 4.1051(0.0467) & 2.7882(0.1179) & 1.3717(0.0251) & 0.7442(0.0150) & 0.4156(0.0336) \\
RELU-MLP & 0.0185(0.0750) & 0.0204(0.0228) & 0.1483(0.1050) & 0.2237(0.3204) & 0.2096(0.9173) & 0.3178(0.9310) \\
Sigmoid-MLP & 0.1619(0.0124)&   0.1339(0.1201)&  0.2271(0.1053)& 0.0677(0.1041)& 0.0887(0.1190)&   0.1450(0.3572) \\
Tanh-MLP &0.0046(0.0011) & 0.0122(0.0061) & 0.0143(0.0039) & 0.0720(0.0881) & 0.2047(0.1995) & 0.2073(0.2639) \\
spline KAN   &  0.0009(0.0013) &    0.0474(0.0541)  &  0.1601(0.1503)  &  0.1859(0.1553)    & 0.1829(0.1946)  &  0.0832(0.0747)\\
hierarchical Tanh-KAN &0.1322(0.1365) & 0.2409(0.0489) & 0.0925(0.0853) & 0.0462(0.0815) & 0.0893(0.0738) & 0.0673(0.0221)\\
\midrule
\multicolumn{7}{c}{\textbf{three-block hierarchical RBF-KAN versus six-layer multilayer perceptrons}} \\
\midrule
\textbf{hierarchical RBF-KAN (ours)} 
& \textbf{0.0076(9.7610)} 
& \textbf{0.0120(6.6664)} 
& \textbf{0.0199(3.3935)} 
& \textbf{0.0229(1.3717)} 
& \textbf{0.0166(0.7442)} 
& \textbf{0.0329(0.4156)} \\
RBF-MLP-I & 0.0249(0.1293) & 0.2842(0.4694) & 0.7559(0.5719) & 0.0459(0.6671) & 0.0269(0.5848) & 0.0346(0.4185) \\
RBF-MLP-II & 0.1364(0.0032) & 0.0682(0.1297) & 0.0653(0.4979) & 0.0620(1.1130) & 0.0487(0.7442) & 0.0437(0.4157) \\
RBF-KAN & 9.7610(7.8114) & 6.6664(0.0467) & 3.3935(0.1179) & 1.3717(0.0251) & 0.7442(0.0150) & 0.4156(0.0336) \\
RELU-MLP & 0.0151(0.0250) & 0.0141(0.0150) & 0.0463(0.0783) & 0.1159(0.1817) & 0.1188(0.7212) & 0.3434(0.7354) \\
Sigmoid-MLP & 0.0390(0.3069) & 0.0354(0.1220) & 0.0244(0.1845) & 0.0706(0.0901) & 0.0868(0.2610) & 0.1515(0.3569)\\
Tanh-MLP & 0.0067(0.0103) & 0.0133(0.0094) & 0.0191(0.0061) & 0.0912(0.0978) & 0.1501(0.1481) & 0.1857(0.2332) \\
spline KAN &     0.0007(0.0030) &   0.1900(0.1354)  &  0.1638(0.2021) &   0.2198(0.4687)  &  0.2026(0.2209) &   0.0808(0.0737) \\ 
hierarchical Tanh-KAN & 0.0686(2.5925) & 0.0847(0.0425) & 0.0272(0.2833) & 0.0339(0.6014) & 0.0229(0.7442) & 0.0700(0.3013)\\
\bottomrule
\end{tabular}
\label{example1_error}
\end{table}}

{\tiny \begin{table}[htbp]
\centering
\caption{Runtime for training different neural networks to learn Eq.~\eqref{example1_model}.}
\label{tab:three_blocks}
\tiny
\begin{tabular}{ccccccc}
\toprule
\multicolumn{7}{c}{\textbf{one-block hierarchical RBF-KAN versus two-layer multilayer perceptrons}} \\
\midrule
Dimension of $\bm{x}$  & 1 & 2 & 3 & 4 & 5 & 6 \\
\midrule
\textbf{hierarchical RBF-KAN (ours)} 
& \textbf{23.2(15.9)} 
& \textbf{40.3(17.0)} 
& \textbf{44.8(22.2)} 
& \textbf{47.5(27.0)} 
& \textbf{54.3(28.8)} 
& \textbf{52.6(21.3)} \\
RBF-MLP-I & 158.3(65.8) & 162.9(87.4) & 170.0(86.9) & 165.2(86.8) & 167.4(133.8) & 162.9(132.7) \\
RBF-MLP-II & 56.7(58.4) & 73.0(77.9) & 73.6(73.4) & 68.6(74.9) & 75.9(74.2) & 76.7(71.6) \\
RBF-KAN & 39.5(15.6) & 42.7(17.3) & 32.8(18.6) & 42.3(20.6) & 45.3(24.2) & 41.0(22.8) \\
RELU-MLP & 23.1(16.0) & 27.3(15.8) & 28.0(15.9) & 27.6(16.0) & 31.8(16.3) & 33.1(16.4) \\
Sigmoid-MLP & 29.2(15.6) & 28.4(16.3) & 28.4(16.7) & 26.8(17.4) & 26.2(17.5) & 26.1(17.3) \\
Tanh-MLP & 23.9(13.9) & 24.2(14.4) & 24.2(14.5) & 24.3(14.6) & 22.7(14.6) & 22.6(14.9) \\
spline KAN & 381.7(343.3) & 730.8(347.7) & 354.0(355.6) & 359.3(399.5) & 363.2(400.1) & 365.3(526.9) \\
hierarchical Tanh-KAN 
& 15.0(13.1) 
& 13.8(13.0) 
& 16.2(14.2) 
& 16.8(14.3) 
& 18.1(14.6) 
& 21.0(17.2) \\
\midrule
\multicolumn{7}{c}{\textbf{two-block hierarchical RBF-KAN versus four-layer multilayer perceptrons}} \\
\midrule
\textbf{hierarchical RBF-KAN (ours)} 
& \textbf{187.5(129.2)} 
& \textbf{208.0(177.1)} 
& \textbf{280.2(267.7)} 
& \textbf{343.0(186.4)} 
& \textbf{382.3(181.8)} 
& \textbf{274.0(249.2)} \\
RBF-MLP-I & 279.4(306.9) & 307.7(317.7) & 274.3(314.9) & 249.2(335.5) & 276.5(354.3) & 229.2(240.5) \\
RBF-MLP-II & 67.5(90.5) & 77.7(107.6) & 77.1(104.2) & 78.5(59.6) & 76.6(67.5) & 72.9(64.4) \\
RBF-KAN & 109.4(68.2) & 107.5(74.3) & 93.3(72.9) & 85.8(72.4) & 100.4(72.1) & 101.2(68.7) \\
RELU-MLP & 45.4(34.9) & 47.7(38.6) & 46.8(37.6) & 48.9(37.0) & 49.5(37.6) & 49.7(39.8) \\
Sigmoid-MLP & 83.5(28.0) & 89.4(29.3) & 88.9(30.1) & 87.3(29.7) & 100.1(28.8) & 111.8(30.4) \\
Tanh-MLP & 58.6(32.5) & 62.0(33.4) & 60.3(35.6) & 59.9(32.7) & 61.5(27.7) & 65.0(28.0) \\
spline KAN & 1101.6(1214.8) & 1472.8(1096.5) & 1330.9(1611.6) & 1027.4(1622.2) & 1094.1(1037.9) & 1021.5(1128.0) \\
hierarchical Tanh-KAN & 38.3(29.0) & 45.0(34.1) & 52.5(41.7) & 63.9(55.9) & 106.6(75.0) & 143.8(109.8) \\
\midrule
\multicolumn{7}{c}{\textbf{three-block hierarchical RBF-KAN versus six-layer multilayer perceptrons}} \\
\midrule
\textbf{hierarchical RBF-KAN (ours)} 
& \textbf{437.6(375.9)} 
& \textbf{394.3(367.9)} 
& \textbf{504.4(383.3)} 
& \textbf{655.9(464.3)} 
& \textbf{571.8(446.2)} 
& \textbf{566.8(488.4)} \\
RBF-MLP-I & 526.2(451.5) & 508.3(425.1) & 502.6(423.0) & 494.8(421.7) & 524.6(432.0) & 509.3(426.1) \\
RBF-MLP-II & 111.2(83.9) & 131.6(97.1) & 135.1(97.5) & 139.4(96.0) & 142.8(97.1) & 112.8(106.3) \\
RBF-KAN & 109.4(68.2) & 107.5(74.3) & 93.3(72.9) & 85.8(72.4) & 100.4(72.1) & 101.2(68.7) \\
RELU-MLP & 66.7(53.5) & 68.6(52.4) & 69.6(48.9) & 70.9(53.8) & 69.1(56.0) & 69.3(56.5) \\
Sigmoid-MLP & 71.5(43.6) & 69.5(43.5) & 67.1(42.9) & 68.6(42.0) & 68.9(42.1) & 72.6(42.7) \\
Tanh-MLP & 60.4(36.2) & 61.7(36.1) & 62.0(36.6) & 63.8(36.6) & 68.9(36.8) & 65.8(36.7) \\
spline KAN & 2005.4(1611.0) & 1839.9(1984.1) & 1905.8(1799.8) & 2570.6(2250.8) & 2005.8(2529.4) & 1569.7(2098.2) \\ 
hierarchical Tanh-KAN & 300.1(242.1) & 334.1(240.0) & 334.6(262.5) & 386.8(288.2) & 427.6(316.2) & 428.9(337.3) \\
\bottomrule
\end{tabular}
\label{example1_cost}
\end{table}}

From Table~\ref{example1_error}, we observe that the proposed hierarchical RBF-KAN maintains high approximation accuracy for the oscillatory function Eq.~\eqref{example1_model} as the dimensionality of the input variable $\bm{x}$ increases from $1$ to $6$. In contrast, the existing RBF-MLP-I and RBF-MLP-II architectures achieve satisfactory accuracy only in the univariate setting. Moreover, the naive RBF-KAN framework fails to accurately learn Eq.~\eqref{example1_model}, suggesting that employing a uniform number of neurons across all layers \citep{chao2026rbf}, which does not fully reproduce the Kolmogorov--Arnold representation structure for continuous functions, is insufficient when using RBF activation functions for learning highly oscillatory multivariate functions. Furthermore, the original KAN architecture employing spline activation functions, as proposed in \cite{liu2024kan}, also fails to accurately reconstruct Eq.~\eqref{example1_model} when the dimensionality \(d>1\).
Among the multilayer perceptron (MLP) architectures with standard activation functions, the \texttt{Tanh} activation performs best; however, its prediction accuracy deteriorates rapidly as the dimensionality $d$ increases. For the proposed hierarchical RBF-KAN, increasing the number of blocks consistently improves approximation accuracy, while the incorporation of the ResNet technique is essential for achieving stable and accurate learning performance. As an additional experiment, we replace the Gaussian-kernel RBF activation in the proposed hierarchical RBF-KAN with the \texttt{Tanh} activation function. Compared with the standard Tanh-MLP, the hierarchical Tanh-KAN yields more accurate reconstructions of Eq.~\eqref{example1_model} in higher-dimensional settings, indicating that the hierarchical structure of our RBF-KAN shown in Fig.~\ref{fig:rbfnn} itself contributes significantly to the learning of multivariate functions. Nevertheless, replacing the RBF activation with the \texttt{Tanh} activation leads to reduced approximation accuracy without providing noticeable improvements in computational runtime. One possible explanation is that the \texttt{Tanh} activation function may require stronger smoothness properties of the target function, making it less suitable for approximating the highly oscillatory multivariate function in Eq.~\eqref{example1_model}.

Fig.~\ref{example1_cost} shows that increasing the number of intermediate layers or blocks in the neural network leads to increased computational cost. The runtime of the proposed hierarchical RBF-KAN is longer than that of multilayer perceptrons with the same effective depth (one block in the hierarchical RBF-KAN contains two activation layers and is therefore comparable to two hidden layers in a multilayer perceptron), but it remains more computationally efficient than the spline-based KAN architecture. Moreover, replacing the RBF activation function with the \texttt{Tanh} activation function within the hierarchical RBF-KAN framework does not result in an apparent reduction in runtime. This observation suggests that the increased computational cost relative to standard multilayer perceptrons primarily arises from the hierarchical structure itself, rather than from the use of RBF activation functions.

We also conduct an additional sensitivity analysis on the two key hyperparameters of the hierarchical RBF-KAN: the number of RBFs used to approximate $\phi_{q,p}$ and $\Phi_q$ in the Kolmogorov-Arnold representation Eq.~\eqref{KAN_expression}, denoted by $\ell$, and the number of neurons in the output of each intermediate block, denoted by $n_0$. The corresponding results are presented in Appendix~\ref{appendixB}. We find that choosing a excessively small value of the number of neurons for the intermediate blocks $n_0\leq 4$ leads to inaccurate reconstruction of Eq.~\eqref{example1_model}, whereas increasing $\ell$ improves approximation accuracy. However, increasing either $n_0$ or $\ell$ also increases the computational cost of the proposed hierarchical RBF-KAN.
\end{example}

Next, we carry out an additional test on comparing our proposed Hierarchical RBF-KAN against other neural network architectures for learning a dynamical system. 

\begin{example}
    \rm
    \label{example2}
Consider the Lorenz system, which exhibits chaotic dynamics \citep{hirsch2013differential,lorenz2017deterministic}:
\begin{equation}
\begin{cases}
\frac{\d x}{\d t} = \sigma (y - x), \\
\frac{\d y}{\d t} = x(\rho - z) - y, \\
\frac{\d z}{\d t} = xy - \beta z.
\end{cases}
\label{example2_model}
\end{equation}
We reconstruct Eq.~\eqref{example2_model} using a neural network of the form
\begin{equation}
    \frac{\d (\hat x, \hat y, \hat z)}{\d t}
    =
    \bigl(f_1(\hat{x},\hat{y},\hat{z}),\, f_2(\hat{x},\hat{y},\hat{z}),\, f_3(\hat{x},\hat{y},\hat{z})\bigr),
    \label{example2_approximate}
\end{equation}
where $f_1$, $f_2$, and $f_3$ denote the outputs of the neural network which takes $(\hat x,\hat y,\hat z)$ as its inputs.

To generate the training data, we sample the initial condition according to
\[
(x_0,x_1,x_2)^T \sim \mathcal{N}\bigl((-10,-10,25)^T,\,0.05I_3\bigr),
\]
and generate 30 independent sets of trajectories on the time interval $t\in[0,2]$ with time step $\Delta t=0.04$. Another 30 trajectory sets sampled from the same distribution are used for testing. We use the \texttt{torchdiffeq} package to numerically solve Eqs.~\eqref{example2_model} and~\eqref{example2_approximate}, and to perform backpropagation through the ODE solver for parameter optimization. The neural network is trained by minimizing the mean squared trajectory error:
\begin{equation}
    \frac{1}{T+1}\frac{1}{N}\sum_{i=0}^T\sum_{j=1}^N
    \left\|
    \big(x_j(t_i),y_j(t_i),z_j(t_i)\big)
    -
    \big(\hat x_j(t_i),\hat y_j(t_i),\hat z_j(t_i)\big)
    \right\|_2^2.
\end{equation}

To evaluate the accuracy of the learned ODE system~\eqref{example2_approximate}, we compute the average relative errors in both the predicted trajectories and the predicted dynamics:
\begin{equation}
\begin{aligned}
    &\text{Relative error in predicted trajectories} \\
    &\hspace{0.5cm}\coloneqq
    \left(
    \frac{
    \sum_{i=0}^T\sum_{j=1}^N
    \left\|
    \big(x_j(t_i),y_j(t_i),z_j(t_i)\big)
    -
    \big(\hat x_j(t_i),\hat y_j(t_i),\hat z_j(t_i)\big)
    \right\|_2^2
    }{
    \sum_{i=0}^T\sum_{j=1}^N
    \left\|
    \big(x_j(t_i),y_j(t_i),z_j(t_i)\big)
    \right\|_2^2
    }
    \right)^{1/2},
    \\[2mm]
    &\text{Relative error in predicted dynamics} \\
    &\hspace{0.5cm}\coloneqq
    \left(
    \frac{
    \sum_{i=0}^T\sum_{j=1}^N
    \left\|
    \left(
    \frac{\d x_j}{\d t}(t_i),
    \frac{\d y_j}{\d t}(t_i),
    \frac{\d z_j}{\d t}(t_i)
    \right)
    -
    \left(
    \frac{\d \hat x_j}{\d t}(t_i),
    \frac{\d \hat y_j}{\d t}(t_i),
    \frac{\d \hat z_j}{\d t}(t_i)
    \right)
    \right\|_2^2
    }{
    \sum_{i=0}^T\sum_{j=1}^N
    \left\|
    \left(
    \frac{\d x_j}{\d t}(t_i),
    \frac{\d y_j}{\d t}(t_i),
    \frac{\d z_j}{\d t}(t_i)
    \right)
    \right\|_2^2
    }
    \right)^{1/2}.
\end{aligned}
\label{example2_metric}
\end{equation}

    \begin{figure}
    \centering
\includegraphics[width=0.9\linewidth]{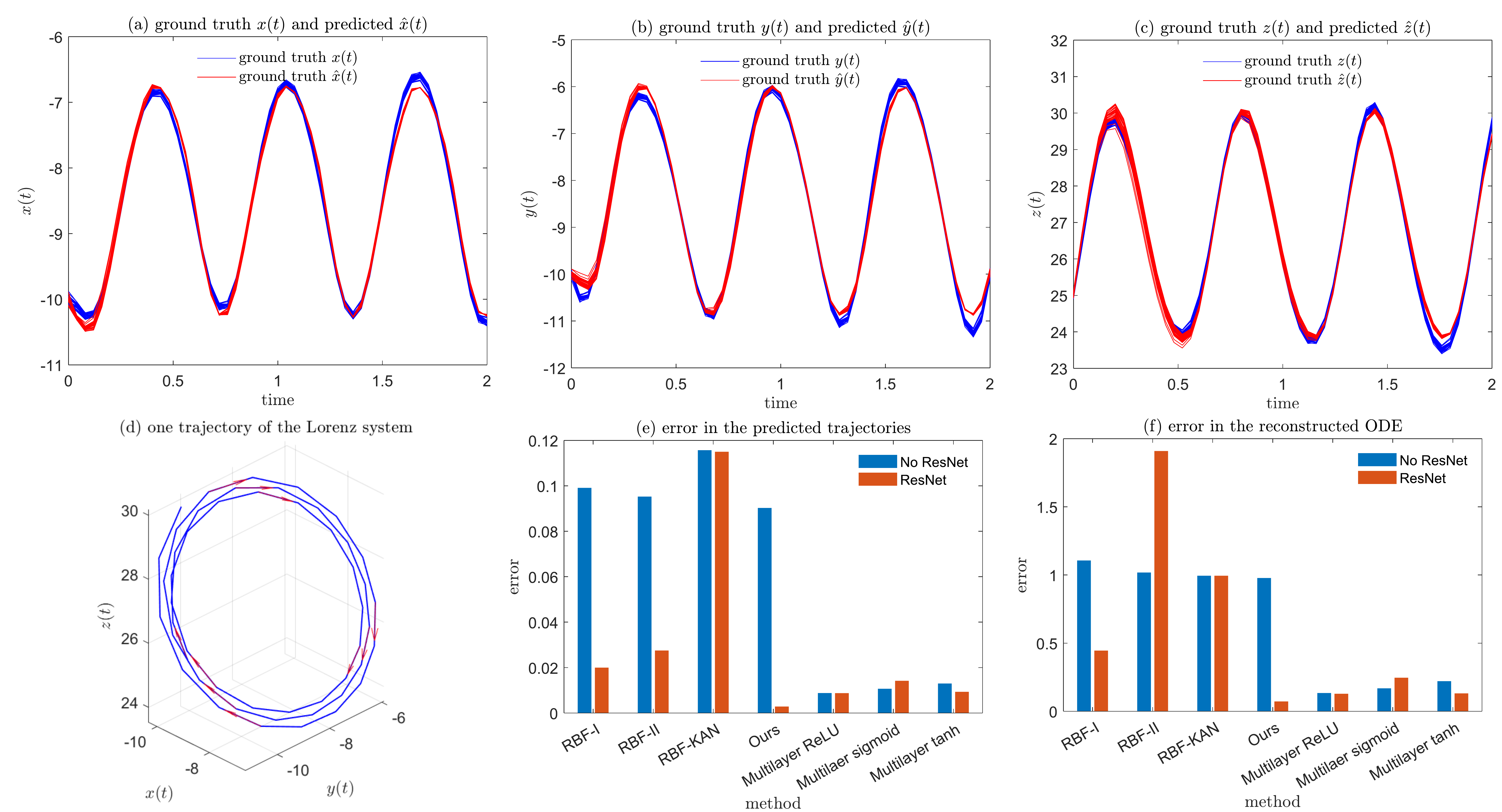}
    \caption{(a)--(c) Comparison between the ground-truth test trajectories $(x(t),y(t),z(t))$ generated by the Lorenz system Eq.~\eqref{example2_model} and the trajectories generated by the learned dynamical system Eq.~\eqref{example2_approximate} using the proposed hierarchical RBF-KAN. (d) A representative trajectory generated by the chaotic Lorenz system Eq.~\eqref{example2_model}. (e)--(f) Average relative errors in the predicted trajectories and learned dynamics, respectively, as defined in Eq.~\eqref{example2_metric}, for different neural network architectures and activation functions (all multilayer perceptrons have four layers).}
    \label{fig:example2}
\end{figure}

As shown in Fig.~\ref{fig:example2}(a)--(c), the proposed hierarchical RBF-KAN with two blocks, equipped with the ResNet technique, accurately reconstructs the Lorenz system Eq.~\eqref{example2_model} and yields reliable trajectory predictions even in the presence of uncertainty in the initial conditions. Both the relative error in the reconstructed trajectories and the relative error in the learned dynamics are the smallest among all considered neural network architectures and activation functions. 
In addition, Fig.~\ref{fig:example2}(e)--(f) shows that several existing RBF-based neural network architectures produce significantly larger errors in the learned dynamics, both with and without the ResNet technique, and in some cases yield larger errors than standard multilayer perceptrons. We again observe that incorporating the ResNet technique is essential for achieving stable and accurate learning within the two-block hierarchical RBF-KAN framework. Overall, the proposed multi-block hierarchical RBF-KAN equipped with the ResNet structure consistently outperforms four-layer multilayer perceptrons using the \texttt{ReLU}, \texttt{sigmoid}, or \texttt{Tanh} activation functions.

\end{example}

Finally, we consider reconstructing a multivariate random field model to evaluate the performance of our proposed hierarchical RBF-SKAN framework, as shown in Fig.~\ref{fig:rbfsnn}.
\begin{example}
    \rm
    \label{example3}

We consider the problem of learning a random field model, which can be viewed as a multivariate stochastic extension of the oscillatory and discontinuous function approximation problem studied in \cite[Subsection~4.1.2]{jiang2022efficient}. The random field is modeled using the proposed hierarchical RBF-SKAN shown in Fig.~\ref{fig:rbfsnn}:
\begin{equation}
    \bm{y}(\bm{x};\bm{\epsilon})
    =
    \bigl(y_1(\bm{x}),\ldots,y_d(\bm{x})\bigr),
    \qquad
    \bm{x}\in\mathbb{R}^d.
\end{equation}
Each component is defined by:
\begin{equation}
   y_j(\bm{x};\bm{\epsilon}) =
\begin{cases}
0.2 \sin\!\left(6\big(\sum_{i=1}^d c_{i,j}x_i\big)\right)
+
\sin\!\left(\sigma\big(\sum_{i=1}^d\epsilon_i\tilde{c}_{i,j}x_i\big)\right),
&
\displaystyle \sum_{i=1}^d c_{i,j}x_i \le 0,
\\[2mm]
1
+
0.1x_j\cos\!\left(12\big(\sum_{i=1}^d c_{i,j}x_i\big)\right)
+
\sin\!\left(\sigma\big(\sum_{i=1}^d \epsilon_i\tilde{c}_{i,j}x_i\big)\right),
&
\displaystyle \sum_{i=1}^d c_{i,j}x_i > 0,
\end{cases}
\label{example3_model}
\end{equation}
where the coefficients $c_{i,j}$ and $\tilde{c}_{i,j}$ are independently sampled from the uniform distribution $\mathcal{U}(1,1.1)$ and remain fixed across all training and testing samples, while \(\epsilon_i\sim\mathcal{N}(0,1)\) are independently sampled for each realization of the random field, and \(\sigma=0.2\) denotes the noise strength.

We employ the hierarchical RBF-SKAN shown in Fig.~\ref{fig:rbfsnn} and train it by minimizing the mini-batch local squared $W_2$ loss function proposed in \cite{xia2025efficient}, which is an effective loss function to minimize for learning random fields under the $W_2$ metric \citep{xia2026local},
with a neighborhood size $\delta=0.1\sqrt{d}$ and a training batch size of $25$. 
To evaluate the reconstruction accuracy of the learned random field model, we compute the average relative errors in the predicted means and standard deviations across different dimensions, averaged over five independent experiments:
\begin{equation}
\begin{aligned}
    &\text{Average relative error in the predicted mean}
    \coloneqq
    \frac{1}{dN_t}
    \sum_{j=1}^{N_t}
    \sum_{i=1}^d
    \frac{
    \left|
    \E[\hat{y}_i(\bm{x}_j;\hat{\bm{\epsilon}})]
    -
    \E[y_i(\bm{x}_j;\bm{\epsilon})]
    \right|
    }{
    \left|
    \E[y_i(\bm{x}_j;\bm{\epsilon})]
    \right|
    },
    \\[2mm]
    &\text{Average relative error in the predicted standard deviation}
    \coloneqq
    \frac{1}{dN_t}
    \sum_{j=1}^{N_t}
    \sum_{i=1}^d
    \frac{
    \left|
    \mathrm{Std}[\hat{y}_i(\bm{x}_j;\hat{\bm{\epsilon}})]
    -
    \mathrm{Std}[y_i(\bm{x}_j;\bm{\epsilon})]
    \right|
    }{
    \left|
    \mathrm{Std}[y_i(\bm{x}_j;\bm{\epsilon})]
    \right|
    },
\end{aligned}
\label{average_error_def}
\end{equation}
where $\hat{\bm{\epsilon}}$ stands for randomness introduced in our hierarchical RBF-SKAN.

{\tiny \begin{table}[htbp]
\centering
\caption{Relative errors (defined in Eq.~\eqref{average_error_def}) in the predicted mean and standard deviation for different methods. Numbers in parentheses refer to the error in the predicted standard deviation. Errors are averaged over 5 repeated experiments.}
\label{tab:three_blocks}
\tiny
\begin{tabular}{ccccccc}
\toprule
Dimension of $\bm{x}$ and $\bm{y}$  & 1 & 2 & 3 & 4 & 5  \\
\midrule                  
\textbf{hierarchical RBF-SKAN (ours)} & \textbf{0.0449(0.2798)}  & \textbf{0.0771(0.2836)}  & \textbf{0.1356(0.2841)}  & \textbf{0.1392(0.2488)}  & \textbf{0.1830(0.2916)} \\
hierarchical RBF-SKAN (deterministic scales) &                0.0474(0.2767) &  0.0776(0.2524)  & 0.1330(0.2690)  & 0.1400(0.2380) &   0.1858(0.3535) \\
CVAE &     0.0180(0.9363)                &  0.0339(0.8998)  &  0.0456(0.9179) &  0.0611(0.9176)  &  0.0666(0.7921) \\
CVAE (RBF) &     0.0185(0.9212)                &  0.0346(0.8776)  &   0.0470(0.9355) &  0.0586(0.9200)  &  0.0664(0.7858) \\
CNF (GELU) & 0.0456(0.3063)  &  0.0740(0.2985) &  0.1299(0.3089)  &  0.1515(0.3316) &  0.1877(0.3573) \\         
CNF (RBF) & 0.0479(0.2815)  &  0.0797(0.2854) &  0.1303(0.3261)  &  0.1498(0.3369) &  0.1919(0.3688) \\
\bottomrule
\end{tabular}
\label{example3_error}
\end{table}}

{\tiny \begin{table}[htbp]
\centering
\caption{Training runtime for different methods (in seconds), averaged over five independent experiments.}
\label{tab:three_blocks}
\tiny
\begin{tabular}{ccccccc}
\toprule
Dimension of $\bm{x}$ and $\bm{y}$  & 1 & 2 & 3 & 4 & 5  \\
\midrule
\textbf{hierarchical RBF-SKAN (ours)} 
& \textbf{541}  
& \textbf{296}  
& \textbf{730}  
& \textbf{2334}  
& \textbf{4238}\\
hierarchical RBF-SKAN (deterministic scales) &    581  &  291 &  827  & 2181 &  4590 \\
CVAE (GELU) &        36.9 &   50.1&   59.4&   67.3 &  79.5\\
CVAE (RBF) &        34.3 &   51.3&   58.7&   67.3 &  76.5\\
CNF (GELU) & 25.5     &  32.3  &  38.7  &  45.0  &  52.7  \\
CNF (RBF) & 24.0     &  31.2  &  36.0   &  46.9  &  46.6  \\
\bottomrule
\end{tabular}
\label{example3_cost}
\end{table}}


We compare our hierarchical RBF-SKAN with the continuous normalizing flow (CNF) approach \citep{papamakarios2017masked, winkler2019learning} and the conditional variational autoencoder (CVAE) approach \citep{lopez2017conditional, kingma2013auto} for reconstructing the random field model Eq.~\eqref{example3_model}. For both the CNF and CVAE frameworks, we consider neural networks equipped with either the \texttt{GELU} activation function or the Gaussian-kernel RBF activation function Eq.~\eqref{Gaussian_kernel}, together with the ResNet technique.
The CVAE approach fails to accurately capture the standard deviation of the random field model, while the average error in the predicted standard deviation for the CNF approach increases as the dimensionality $d$ grows. Furthermore, for both the CVAE and CNF architectures, replacing the \texttt{GELU} activation with the Gaussian-kernel RBF activation Eq.~\eqref{Gaussian_kernel} does not improve the accuracy of the reconstructed means and standard deviations of Eq.~\eqref{example3_model}.
In contrast, the proposed hierarchical RBF-SKAN achieves the smallest error in the predicted standard deviation, particularly in higher-dimensional settings, where its performance remains comparatively stable as $d$ increases. Although the average error in the predicted mean increases moderately with dimensionality, this behavior is primarily attributed to the increasing sparsity of the training data in multidimensional spaces. As an additional experiment, we enforce the scale parameters in the second module of the hierarchical RBF-SKAN shown in Fig.~\ref{fig:rbfsnn} to be deterministic, corresponding to the special case considered in the proof of Theorem~\ref{theorem2}. We observe that introducing stochasticity into the scale parameters does not lead to a noticeable increase in computational cost. However, allowing the scales of the RBFs in the second module of the hierarchical RBF-SKAN to be stochastic improves the accuracy of the predicted standard deviation when \(d=5\).
Finally, the training runtime of the hierarchical RBF-SKAN is longer than that of the CVAE and CNF methods for large $d$. This additional computational cost arises because evaluating the squared \(W_2\) distance between two empirical distributions with \(N\) samples requires \(\mathcal{O}(N^3\log N)\) operations, which is more expensive than the loss evaluations used in the CVAE and CNF frameworks. Moreover, increasing the dimensionality \(d\) results in a more complex architecture for the proposed RBF-SKAN, which may further increase computational runtime.
As a potential future direction, it would be beneficial to replace the local squared $W_2$ loss with an entropy-regularized Sinkhorn distance to further reduce the computational complexity \citep{cuturi2013sinkhorn}.
\end{example}

\section{Summary and conclusion}
\label{section4}

In this manuscript, we proposed a hierarchical RBF-KAN architecture for efficiently learning multidimensional deterministic functions and a hierarchical RBF-SKAN framework for learning random fields. From a theoretical perspective, we analyzed the approximation properties of the proposed hierarchical RBF-KAN and showed that it has the potential to partially alleviate the curse of dimensionality in multivariate function approximation. Furthermore, we established that the proposed hierarchical RBF-SKAN framework possesses the capability to approximate random field models given mild conditions under the Wasserstein metric. Empirically, the proposed hierarchical RBF-KAN consistently outperformed several prevailing RBF-based neural network architectures, including the naive RBF-KAN framework recently proposed in \cite{chao2026rbf}, as well as neural networks employing alternative activation functions. In addition, the proposed hierarchical RBF-SKAN demonstrated better performance compared with existing continuous normalizing flow (CNF) and conditional variational autoencoder (CVAE) approaches for random field learning tasks.

There are several promising directions for future research. First, it would be valuable to investigate how deeper architectures, \textit{i.e.}, increasing the number of blocks in the hierarchical RBF-KAN or hierarchical RBF-SKAN frameworks, could further improve approximation accuracy for multidimensional functions and random fields. Second, integrating the proposed hierarchical RBF-KAN structures as modular building blocks within more sophisticated machine learning architectures, such as transformer-based models, represents another promising direction. Finally, the current hierarchical RBF-KAN and hierarchical RBF-SKAN frameworks are primarily designed for regression-type problems involving continuous outputs. Extending the proposed hierarchical RBF-KAN framework and RBF-SKAN framework to classification tasks and other discrete-output learning problems is therefore an important topic for future investigation.

\printcredits

\section*{Declaration of competing interest}
The authors declare that they have no known competing financial interests or personal relationships that could have appeared to influence the work reported in this paper.

\bibliographystyle{cas-model2-names}

\bibliography{bibliography}

\appendix

\section*{Appendix}
\section{Proof of Theorem~\ref{theorem1}}
\label{proof_theorem1}
Here, we provide proof of Theorem~\ref{theorem1}. Given a scalar-valued continuous function \(u(\bm{x})\), we consider its Kolmogorov--Arnold representation in Eq.~\eqref{KAN_expression}, where the functions \(\phi_{q,p}\) and \(\Phi_q\) satisfy the continuity conditions in Eqs.~\eqref{uniform1}--\eqref{uniform3}. For any \(1>\varepsilon>0\), we first define the mollified function:
\begin{equation}
    \phi_{q,p}^{\varepsilon}(x_p)
    :=
    \int_{\mathbb{R}}
    \tilde{\phi}_{q,p}(y)\,
    B_{\varepsilon^2}(x_p-y)
    \, \d y,
\end{equation}
where
\[
B_{\varepsilon^2}({x})
\coloneqq
\frac{1}{\sqrt{2\pi\varepsilon^4}}
\exp\!\left(-\frac{{x}^2}{2\varepsilon^4}\right)
\]
is the scaled and normalized Gaussian kernel RBF. 
We let $\tilde{\phi}_{q,p}$ be a $C^0$ extension of $\phi_{q,p}$ with compact support, defined by
\begin{equation}
\tilde{\phi}_{q,p}(x_p)=
\left\{
\begin{aligned}
    &\phi_{q,p}(x_p), && x_p\in[-1,1],\\[2mm]
    &\phi_{q,p}(1), && x_p\in[1,1+\varepsilon],\\[2mm]
    &\frac{\left(\left|\frac{\phi_{q,p}(1)}{\delta_{\phi_{q,p}}(\varepsilon)/\varepsilon}\right|-x_p+1+\varepsilon\right)}
    {\left|\frac{\phi_{q,p}(1)}{\delta_{\phi_{q,p}}(\varepsilon)/\varepsilon}\right|}
    \,\phi_{q,p}(1), && x_p\in[1+\varepsilon,\,x_{q,p}^2],\\[2mm]
    &\phi_{q,p}(-1), && x_p\in[-1-\varepsilon,-1],\\[2mm]
    &\frac{\left(\left|\frac{\phi_{q,p}(-1)}{\delta_{\phi_{q,p}}(\varepsilon)/\varepsilon}\right|+x_p+1+\varepsilon\right)}
    {\left|\frac{\phi_{q,p}(-1)}{\delta_{\phi_{q,p}}(\varepsilon)/\varepsilon}\right|}
    \,\phi_{q,p}(-1), && x_p\in[-x_{q,p}^1,\,-1-\varepsilon],\\[2mm]
    &0, && x_p>x_{q,p}^2 \ \text{or}\ x_p<-x_{q,p}^1,
\end{aligned}
\right.
\end{equation}
where 
\begin{equation}
    x_{q,p}^1 \coloneqq \left|\frac{\phi_{q,p}(-1)}{\delta_{\phi_{q,p}}(\varepsilon)/\varepsilon}\right| + \varepsilon + 1,
    \qquad
    x_{q,p}^2 \coloneqq \left|\frac{\phi_{q,p}(1)}{\delta_{\phi_{q,p}}(\varepsilon)/\varepsilon}\right| + \varepsilon + 1.
\end{equation}
Then, $\tilde{\phi}_{q,p}(x_p)$ satisfies
\[
\|\tilde{\phi}_{q,p}\|_{\infty}\leq \|\phi_{q,p}\|_{\infty},
\qquad
\delta_{\phi_{q,p}}(\varepsilon)\geq \delta_{\tilde{\phi}_{q,p}}(\varepsilon).
\]
Moreover, let $x_{q,p}\coloneqq \max\{x_{q,p}^1,x_{q,p}^2\}$.
Then $\tilde{\phi}_{q,p}(x_p)$ is compactly supported on $(-x_{q,p},x_{q,p})$. Furthermore, we can verify that
\begin{equation}
    |\phi_{q,p}^{\varepsilon}(x_p)|
    \leq
    B_{\varepsilon^2}(|x_p|-x_{q,p})\, 2x_{q,p}\|\phi_{q,p}\|_{\infty},
    \qquad |x_p|>x_{q,p}.
\label{out_bound}
\end{equation}

Additionally, for any $x_p\in\mathbb{R}$, we have
\begin{equation}
\begin{aligned}
    \big|\phi_{q,p}^{\varepsilon}(x_p)-\tilde{\phi}_{q,p}(x_p)\big|
    &\leq
    \left|
    \tilde{\phi}_{q,p}(x_p)
    -
    \int_{\mathbb{R}}
    \tilde{\phi}_{q,p}(y)\,
    B_{\varepsilon^2}(x_p-y)\,\d y
    \right| \\
    &\leq
    2\|\phi_{q,p}\|_{\infty}\bigl(1-\Psi(\varepsilon^{-1})\bigr)
    + \delta_{\phi_{q,p}}(\varepsilon),
\end{aligned}
\label{rbf_error_clean}
\end{equation}
where 
\[
\Psi(r)
:=
\int_{-r}^r
B_{1}({z})\, \d{z}
\]
denotes the Gaussian mass inside $[-r, r]$. The first term in Eq.~\eqref{rbf_error_clean} corresponds to the local smoothing error, while the second term accounts for the Gaussian tail truncation. Furthermore, it can be verified that
\[
\delta_{\phi_{q,p}^{\varepsilon}}(\varepsilon)\leq \delta_{\phi_{q,p}}(\varepsilon)
\qquad\text{and}\qquad
\|\phi_{q,p}^{\varepsilon}\|_{\infty}\leq \|\phi_{q,p}\|_{\infty}.
\]
Next, we define
\begin{equation}
    \phi_{q,p}^{\varepsilon,2}(x_p)
    :=
    \int_{\mathbb{R}}
    \phi_{q,p}^{\varepsilon}(y)\,
    B_{\varepsilon^2}(x_p-y)
    \, \d y.
\end{equation}
Therefore, we have
\begin{equation}
\begin{aligned}
    \big|\phi_{q,p}^{\varepsilon,2}(x_p) - \tilde{\phi}_{q,p}(x_p)\big|
    &\leq
    \big|\phi_{q,p}^{\varepsilon}(x_p) - \tilde{\phi}_{q,p}(x_p)\big|
    +
    \big|\phi_{q,p}^{\varepsilon,2}(x_p) - \phi_{q,p}^{\varepsilon}(x_p)\big| \\
    &\leq
    4\|\phi_{q,p}\|_{\infty}\bigl(1-\Psi(\varepsilon^{-1})\bigr)
    + 2\delta_{\phi_{q,p}}(\varepsilon).
\end{aligned}
\label{estimate1}
\end{equation}

Let the scaled one-dimensional Clenshaw--Curtis points be defined as
\begin{equation}
    X_{m_i}
    :=
    \left\{
    x_{q,p,j}
    =
    -(x_{q,p}+1)\cos\!\left(\frac{\pi (j-1)}{m_i-1}\right),
    \quad j=1,\dots,m_i
    \right\},
\end{equation}
with $m_i = 2^{i-1}+1$ for $i>1$. Using the Smolyak grid construction (see \cite{barthelmann2000high}), we let $I_N$ denote the one-dimensional interpolation operator with $N$ collocation points. Then, for any $v \in F_k$ and $x_p \in (-x_{q,p}-1,\, x_{q,p}+1)$, it follows from \cite[Theorem~8]{barthelmann2000high} that
\begin{equation}
\|v - I_N v\|_{\infty}
\le
c_k \bigl(2(x_{q,p}+1)\bigr)^k
N^{-k/2}\log N\,
\|v\|_{\infty,k}.
\label{sparse_est}
\end{equation}

For each $\phi_{q,p}^{\varepsilon}(x_p)$ and any $x_p\in[-x_{q,p},x_{q,p}]$, we approximate it by
\begin{equation}
    \phi_{q,p,N}(x_p)
    :=
    \sum_{i=1}^N
    w_i\,
    \phi_{q,p}^{\varepsilon}(x_{q,p,i})\,
    B_{\varepsilon^2}(x_p-x_{q,p,i}),
\end{equation}
where $w_i$ are the quadrature weights associated with the grid, chosen so that
\begin{equation}
    \sum_{i=1}^N
    w_i\,
    \phi_{q,p}^{\varepsilon}(x_{q,p,i})\,
    B_{\varepsilon^2}(x_p-x_{q,p,i})
    =
    \int_{[-x_{q,p}-1,\;x_{q,p}+1]}
    I_N\!\left(\phi_{q,p}^{\varepsilon}(y)B_{\varepsilon^2}(x_p-y)\right)\d y,
\end{equation}
for any $x_p\in[-x_{q,p},x_{q,p}]$. Let $x_p\in[-x_{q,p},x_{q,p}]$. We decompose the approximation error as
\begin{equation}
\begin{aligned}
    \big|\phi_{q,p}(x_p)-\phi_{q,p,N}(x_p)\big|
    \leq\;
    \big|\phi_{q,p}(x_p)-\phi_{q,p}^{\varepsilon,2}(x_p)\big|
    +
    \big|\phi_{q,p}^{\varepsilon,2}(x_p)-\phi_{q,p,N}(x_p)\big|.
\end{aligned}
\label{error_bound}
\end{equation}
The first term in Eq.~\eqref{error_bound} is controlled by Eq.~\eqref{rbf_error_clean}. For the second term, plugging in the inequalities~\eqref{sparse_est} and \eqref{out_bound}, we have: 
\begin{equation}
\begin{aligned}
    \big|\phi_{q,p}^{\varepsilon,2}(x_p)-\phi_{q,p,N}(x_p)\big|
    &\leq
    \Big|
    \int_{\mathbb{R}}
    \phi_{q,p}^{\varepsilon}(y)B_{\varepsilon^2}(x_p-y)\d y
    -
    \sum_{i=1}^N
    w_i\,
    \phi_{q,p}^{\varepsilon}(x_{q,p,i})\,
    B_{\varepsilon^2}(x_p-x_{q,p,i})
    \Big| \\
    &=
    \big|
    \int_{\mathbb{R}}
    \phi_{q,p}^{\varepsilon}(y)B_{\varepsilon^2}(x_p-y)\d y
    -
    \int_{[-x_{q,p}-1,\;x_{q,p}+1]}
    I_N\!\left(\phi_{q,p}^{\varepsilon}(y)B_{\varepsilon^2}(x_p-y)\right)\d y
    \big| \\
    &\leq
    c_{d,k}\,2^k\bigl(2(x_{q,p}+1)\bigr)^k
    N^{-k/2}\log N\,
    \|\phi_{q,p}^{\varepsilon}\|_{\infty,k}\,
    \|B_{\varepsilon^2}\|_{\infty,k} 
    \\&\qquad+
    \int_{\mathbb{R}\setminus[-x_{q,p}-1,\;x_{q,p}+1]}
    \phi_{q,p}^{\varepsilon}(y)B_{\varepsilon^2}(x_p-y)\d y \\
    &\leq
    c_{d,k}\bigl(4(x_{q,p}+1)\bigr)^k
    N^{-k/2}\log N\,
    \|\phi_{q,p}^{\varepsilon}\|_{\infty,k}\,
    \|B_{\varepsilon^2}\|_{\infty,k}
    +
    2x_{q,p}\bigl(1-\Psi(\varepsilon^{-1})\bigr)\|\phi_{q, p}^{\varepsilon}\|_{\infty}.
\end{aligned}
\label{ineq_deterministic}
\end{equation}

Using the Gaussian scaling property, we have
\[
\|B_{\varepsilon^2}\|_{\infty,k}
=
\varepsilon^{-2k}\|B_{1}\|_{\infty,k}.
\]
Furthermore,
\begin{equation}
\|\phi_{q,p}^{\varepsilon}\|_{\infty,k}
\leq
\|\phi_{q,p}\|_{\infty}\,\varepsilon^{2k}
\int_{\mathbb{R}}
\left|
\frac{\d^k B_{1}}{\d x^k}
\right|
\d x.
\end{equation}

We define
\[
\tilde{c}_{d,k}
\coloneqq
c_{d,k}\int_{\mathbb{R}}
\left|
\frac{\d^k B_{1}}{\d x^k}
\right|
\d x,
\]
where $c_{d,k}$ is the constant appearing in Eq.~\eqref{ineq_deterministic}. Therefore, we obtain
\begin{equation}
\begin{aligned}
\big|\phi_{q,p}^{\varepsilon,2}(x_p)-\phi_{q,p,N}(x_p)\big|
&\leq
\tilde c_{d,k}\,
\bigl(4(x_{q,p}+1)\bigr)^k
N^{-k/2}\varepsilon^{4k}
\log N\,
\|\phi_{q,p}\|_{\infty}\,
\|B_{1}\|_{\infty,k}^2 +
2x_{q,p}\bigl(1-\Psi(\varepsilon^{-1})\bigr)\,
\|\phi_{q,p}\|_{\infty}.
\end{aligned}
\label{estimate2}
\end{equation}

Combining Eqs.~\eqref{estimate1} and~\eqref{estimate2}, for any $x_p\in[-1,1]$, we obtain
\begin{equation}
\begin{aligned}
    \big|\phi_{q,p}(x_p)-\phi_{q,p,N}(x_p)\big|
    &\leq
    4\|\phi_{q,p}\|_{\infty}\bigl(1-\Psi(\varepsilon^{-1})\bigr)
    +2\delta_{\phi_{q,p}}(\varepsilon) \\
    &\hspace{-2cm}
    +\tilde c_{d,k}\bigl(4(x_{q,p}+1)\bigr)^k
    N^{-k/2}\varepsilon^{4k}\log N\,
    \|\phi_{q,p}\|_{\infty}\,\|B_{1}\|_{\infty,k} 
    +2x_{q,p}\bigl(1-\Psi(\varepsilon^{-1})\bigr)\|\phi_{q,p}\|_{\infty}.
\end{aligned}
\label{phi_bound}
\end{equation}
Therefore, as $\varepsilon\to 0$ and $N>N_{\varepsilon}\coloneqq \varepsilon^{-10}\to\infty$, the right-hand side of~\eqref{phi_bound} converges to $0$ for all $x_p\in(-1,1)$. In particular, we choose the pair $(\varepsilon,N)$ such that:
\begin{equation}
    \sum_{p=1}^d
    \big|\phi_{q,p}(x_p)-\phi_{q,p,N}(x_p)\big|
    < 1,
    \qquad \forall\,p,q.
\end{equation}
We denote $\tilde{a}_q \coloneqq \max_{1\leq p\leq d}\sum_{p=1}^d |\phi_{q,p}(x_p)|$.
For each $\Phi_q(x)$, without loss of generality, we assume that it has compact support on some interval $(-a_q,a_q)$ with $a_q>\tilde{a}_q$. 
Therefore, there exists an RBF approximation $\Phi_{q,N}(x)$ such that, for any $x\in(-a_q-1,a_q+1)$,
\begin{equation}
\begin{aligned}
|{\Phi}_q(x)-\Phi_{q,N}(x)|
&\leq
4\|\Phi_q\|_{\infty}\bigl(1-\Psi(\varepsilon^{-1})\bigr)
+2\delta_{\Phi_q}(\varepsilon) \\
&\hspace{0.5cm}
+\tilde c_{d,k}\bigl(4(a_q+2)\bigr)^k
\varepsilon^{4k} N^{-k/2}\log N\,
\|\Phi_q\|_{\infty}\,\|B_{1}\|_{\infty,k}
+2(a_q+1)\bigl(1-\Psi(\varepsilon^{-1})\bigr)\|\Phi_q\|_{\infty}.
\end{aligned}
\end{equation}

Finally, for $\bm{x}\coloneqq (x_1,\ldots,x_d)\in[0,1]^d$, we have:
\begin{equation}
\begin{aligned}
    &\left|u(\bm{x})-\sum_{q=0}^{2d}\Phi_{q,N}\!\left(\sum_{p=1}^d \phi_{q,p,N}(x_p)\right)\right| \\
    &\leq
    \left|
    \sum_{q=0}^{2d}\Phi_q\!\left(\sum_{p=1}^d \phi_{q,p,N}(x_p)\right)
    -
    \sum_{q=0}^{2d}\Phi_{q,N}\!\left(\sum_{p=1}^d \phi_{q,p,N}(x_p)\right)
    \right|  +
    \left|
    \sum_{q=0}^{2d}\Phi_q\!\left(\sum_{p=1}^d \phi_{q,p}(x_p)\right)
    -
    \sum_{q=0}^{2d}{\Phi}_q\!\left(\sum_{p=1}^d \phi_{q,p,N}(x_p)\right)
    \right| \\
    &\leq
    \sum_{q=0}^{2d}
    \delta_{\Phi_q}\!\left(\sum_{p=1}^d \left|\phi_{q,p}(x_p)-\phi_{q,p,N}(x_p)\right|\right)  +
    \sum_{q=0}^{2d}
    \Bigg(
    4\|\Phi_q\|_{\infty}\bigl(1-\Psi(\varepsilon^{-1})\bigr)
    +2\delta_{\Phi_q}(\varepsilon) \\
    &\qquad\qquad
    +\tilde c_{d,k}\bigl(4(a_q+2)\bigr)^k \varepsilon^{4k} N^{-k/2}\log N\,
    \|\Phi_q\|_{\infty}\,\|B_{1}\|_{\infty,k}+2(a_q+1)\bigl(1-\Psi(\varepsilon^{-1})\bigr)\|\Phi_q\|_{\infty}
    \Bigg).
\end{aligned}
\label{thm1_bound}
\end{equation}

For any fixed $\bm{x}\in(0,1)^d$, as $\varepsilon\to 0$ and $N>N_{\varepsilon}\coloneqq \varepsilon^{-10}\to\infty$, the last two terms on the right-hand side of Eq.~\eqref{thm1_bound} converge uniformly to zero. Moreover, $a_q(N,\varepsilon)\to a_q\coloneqq \sup_{x_p}\sum_{p=1}^{2d+1} \phi_{q,p}(x_p)$ uniformly. In addition, for each $q$, we have
\[
\sum_{p=1}^d \left|\phi_{q,p}(x_p)-\phi_{q,p,N}(x_p)\right|\to 0
\]
as $\varepsilon\to 0$ and $N>N_{\varepsilon}\coloneqq \varepsilon^{-10}\to\infty$. Note that the approximation $\sum_{q=0}^{2d}\Phi_{q,N}\!\left(\sum_{p=1}^d \phi_{q,p,N}(x_p)\right)$ can be represented using a one-block hierarchical RBF-KAN in Fig.~\ref{fig:rbfnn} with $\ell=N$ (redundant coefficients in the dense linear layers of the hierarchical RBF-KAN can be set to 0).
This completes the proof of Theorem~\ref{theorem1}.

\section{Proof of Theorem~\ref{theorem2}}
\label{proof_theorem2}
For an integer \(N\ge 1\), let $h := \frac{2}{N}$ be the side length of each cell, and let \(\bm{k}=(k_1,\dots,k_d)\) with \(k_i\in\{1,\dots,N\}\) denote a multi-index. We partition \([-1,1]^d\) into \(N^d\) closed cubes
\[
I_{\bm{k}}
:=\prod_{i=1}^d \Big[-1+(k_i-1)h,\; -1+k_i h\Big].
\]
Fix a parameter \(\varepsilon>0\) such that \(\varepsilon<\tfrac{h}{2}\). Then, for each cell \(I_{\bm{k}}\), there exists a smooth \(u_{\bm{k}}^{\varepsilon}\in C^\infty(\mathbb{R}^d)\) of the indicator function \(\mathbb{I}_{\bm{x}\in I_{\bm{k}}}\) such that:
\begin{enumerate}
    \item \textbf{Interior value.} \(u_{\bm{k}}^{\varepsilon}(\bm{x}) = 1-\varepsilon\) for all \(\bm{x}\in I_{\bm{k}}\) satisfying \(\operatorname{dist}(\bm{x},\partial I_{\bm{k}})\ge \varepsilon\).
    \item \textbf{Support.} \(u_{\bm{k}}^{\varepsilon}(\bm{x}) = 0\) for all \(\bm{x}\notin I_{\bm{k}}\).
    \item \textbf{Range.} \(0 \le u_{\bm{k}}^{\varepsilon}(\bm{x}) \le 1-\varepsilon\) for all \(\bm{x}\in\mathbb{R}^d\).
    \item \textbf{Gradient bound.} There exists a constant \(C>0\), independent of \(N\) and \(\varepsilon\), such that
    \[
    \|\nabla u_{\bm{k}}^{\varepsilon}\|_{L^\infty(\mathbb{R}^d)} \le \frac{C}{\varepsilon}.
    \]
\end{enumerate}


Then, by Corollary~\ref{col1}, for each $\bm{k}$, there exists a hierarchical RBF-KAN, denoted as $R^1_{\bm{k}}$, whose outputs are denoted by \(z_{\bm{k}}(\bm{x})\) such that
\[
\|u_{\bm{k}}^{\varepsilon}(\bm{x})-z_{\bm{k}}(\bm{x})\|_{\infty}\leq \varepsilon.
\]

For each \(\bm{k}\), let \(f_{\bm{x}_{\bm{k}}}(\bm{y})\) denote the probability density function of \(\bm{y}(\bm{x}_{\bm{k}};\omega)\). By the generalized Moser theorem \citep{greene1979diffeomorphisms}, suppose \(\bm{\omega}_{\bm{k}}\sim \mathcal{U}((0,1)^n)\) and \(\bm{y}(\bm{x}_{\bm{k}};\omega)\in(0,1)^n\). Then, there exists a diffeomorphism
$\psi_{\bm{k}}:[0, 1]^n\to[0, 1]^n$
such that
\[
f_{\bm{x}_{\bm{k}}}\bigl(\psi_{\bm{k}}(\bm{\omega}_{\bm{k}})\bigr)
\left|
\det\!\left(
\frac{\mathrm{D}\psi_{\bm{k}}(\bm{\omega}_{\bm{k}})}
{\mathrm{D}\bm{\omega}_{\bm{k}}}
\right)
\right|
=
1,
\]
where $\frac{\mathrm{D}\psi_{\bm{k}}}
{\mathrm{D}\bm{\omega}_{\bm{k}}}$
denotes the Jacobian matrix of \(\psi_{\bm{k}}\).
Consequently, the random variable \(\psi_{\bm{k}}(\bm{\omega}_{\bm{k}})\) has probability density function \(f_{\bm{x}_{\bm{k}}}(\cdot)\), which coincides with the probability density function of \(\bm{y}(\bm{x}_{\bm{k}};\omega)\).

The Kolmogorov--Arnold representation in Eq.~\eqref{KAN_expression} holds for each component of the diffeomorphism \(\psi_{\bm{k}}(\bm{\omega}_{\bm{k}})\) for every \(\bm{k}\). Then, by Corollary~\ref{col1}, for any \(\epsilon_0>0\) and any \(\bm{k}\), there exists a vector-valued hierarchical RBF-KAN, denoted by \(R_{\bm{k}}^2\), with output \(\hat{\psi}_{\bm{k}}\) approximating \(\psi_{\bm{k}}\) such that:
\begin{equation}
\label{component_uniform_error}
\bigl\|
\psi_{\bm{k}}(\bm{\omega})
-
\hat{\psi}_{\bm{k}}(\bm{\omega})
\bigr\|_{\infty}
\leq
\epsilon_0.
\end{equation}
Therefore, $\|\hat{\psi}_{\bm{k}}(\omega)\|_{\infty}\leq1+\epsilon_0$.

Given any \(\bm{x}\in I_{\bm{k}}\) satisfying
$\operatorname{dist}(\bm{x},\partial I_{\bm{k}})\ge \varepsilon$, consider the measurable mapping
$\bm{\omega}
\mapsto
\bigl(
\psi_{\bm{k}}(\bm{\omega}),
\hat{\psi}_{\bm{k}}(\bm{\omega})
\bigr)
\in
\mathbb{R}^n\times\mathbb{R}^n$. Let \(\nu\) denote the probability law of \(\bm{\omega}\). We define the coupling (pushforward) measure \(\pi\) on \(\mathbb{R}^n\times\mathbb{R}^n\) by
\[
\pi(A\times \hat{A})
:=
\nu
\Bigl(
\big\{
\bm{\omega}
:
\psi_{\bm{k}}(\bm{\omega})\in A,
\;
\hat{\psi}_{\bm{k}}(\bm{\omega})\in \hat{A}
\big\}
\Bigr),
\]
for Borel sets \(A,\hat{A}\subset\mathbb{R}^n\). By construction, the marginals of \(\pi\) are precisely the probability laws of \(\psi_{\bm{k}}(\bm{\omega})\) and \(\hat{\psi}_{\bm{k}}(\bm{\omega})\), which we denote by \(f_{\bm{x}}\) and \(\hat{f}_{\bm{x}}\), respectively.

Using \(\pi\) as an admissible coupling, the squared \(W_2\) distance satisfies
\begin{equation}
\begin{aligned}
W_2^2(f_{\bm{x}_{\bm{k}}},\hat{f}_{\bm{x}_{\bm{k}}})
&\leq
\int_{\mathbb{R}^n\times\mathbb{R}^n}
\|\bm{y}-\hat{\bm{y}}\|^2
\,\pi(\d\bm{y},\d\hat{\bm{y}}) =
\int_{[0,1]^n}
\bigl\|
\psi_{\bm{k}}(\bm{\omega})
-
\hat{\psi}_{\bm{k}}(\bm{\omega})
\bigr\|^2
\,\d\bm{\omega} \leq
n\epsilon_0^2.
\end{aligned}
\label{W2_basic}
\end{equation}
For any $\bm{x}\in I_{\bm{k}, \epsilon} \coloneqq \prod_{i=1}^d \Big[-1+(k_i-1)h+\varepsilon,\; -1+k_i h-\varepsilon\Big]$, we have:
\begin{equation}
\begin{aligned}
    \E_{\bm{\omega}_{\bm{k}} }
    \big[
    \|\psi(\bm{\omega}_{\bm{k}})
    -
    z_{\bm{k}}(\bm{x})\hat{\psi}(\bm{\omega}_{\bm{k}})
    \|^2
    \big]
    &\leq
    2\E_{\bm{\omega}_{\bm{k}} }
    \big[
    \|
    z_{\bm{k}}(\bm{x})\psi(\bm{\omega}_{\bm{k}})
    -
    z_{\bm{k}}(\bm{x})\hat{\psi}(\bm{\omega}_{\bm{k}})
    \|^2
    \big]
    +
    2\E_{\bm{\omega}_{\bm{k}} }
    \big[
    \|
    \psi(\bm{\omega}_{\bm{k}})
    -
    z_{\bm{k}}(\bm{x})\psi(\bm{\omega}_{\bm{k}})
    \|^2
    \big]\\
    &\leq
    2(1+\varepsilon)^2 n\epsilon_0^2
    +
    2\E_{\bm{\omega}_{\bm{k}} }
    \big[
    \|
    \psi(\bm{\omega}_{\bm{k}})
    -
    z_{\bm{k}}(\bm{x})\psi(\bm{\omega}_{\bm{k}})
    \|^2
    \big]\\
    &\leq
    2(1+\varepsilon)^2 n\epsilon_0^2
    +
    2n\|\bm{y}_{\bm{x}}\|_{\infty}^2\varepsilon^2.
\end{aligned}
\label{intermediate1}
\end{equation}

Finally, by stacking \(\{R_{\bm{k}}^1\}\) as the first building block and \(\{R_{\bm{k}}^2\}\) as the second building block for all $\bm{k}$, we can obtain an RBF-SKAN with the structure shown in Fig.~\ref{fig:rbfsnn}, whose output is given by
\begin{equation}
    \hat{\bm{y}}_{\bm{x}}
    \coloneqq
    \sum_{\bm{k}}
    z_{\bm{k}}(\bm{x})
    \hat{\psi}_{\bm{k}}(\bm{\omega}_{\bm{k}})
\end{equation}
with the associated probability density function denoted by \(\hat{f}_{\bm{x}}\). For each $\bm{x}\in I_{\bm{k},\varepsilon}$, consider the coupling
\[
(\bm{y},\hat{\bm{y}})
\sim
\left(
\psi_{\bm{k}}(\bm{\omega}_{\bm{k}}),
\;
\sum_{\bm{j}}
z_{\bm{j}}(\bm{x})
\hat{\psi}_{\bm{j}}(\bm{\omega}_{\bm{j}})
\right).
\]
Then, using the inequality~\eqref{intermediate1}, we have
\begin{equation}
\begin{aligned}
    W_2^2(f_{\bm{x}},\hat{f}_{\bm{x}})
    &\leq
    2W_2^2(f_{\bm{x}_{\bm{k}}},\hat{f}_{\bm{x}})
    +
    2W_2^2(f_{\bm{x}},f_{\bm{x}_{\bm{k}}})
    \\
    &\leq
    2\E\!\Big[
    \big\|
    \psi_{\bm{k}}(\bm{\omega}_{\bm{k}})
    -
    \sum_{\bm{j}}
    z_{\bm{j}}(\bm{x})
    \hat{\psi}_{\bm{j}}(\bm{\omega}_{\bm{j}})
    \big\|^2
    \Big]
    +
    2L^2h^2
    \\
    &\leq
    4\E\!\Big[
    \big\|
    \psi_{\bm{k}}(\bm{\omega}_{\bm{k}})
    -
    z_{\bm{k}}(\bm{x})
    \hat{\psi}_{\bm{k}}(\bm{\omega}_{\bm{k}})
    \big\|^2
    \Big]
    +
    4\E\!\Big[
    \big\|
    \sum_{\bm{j}\neq \bm{k}}
    z_{\bm{j}}(\bm{x})
    \hat{\psi}_{\bm{j}}(\bm{\omega}_{\bm{j}})
    \big\|^2
    \Big]
    +
    2L^2h^2
    \\
    &\leq
    4\Big(
    2(1+\varepsilon)^2 n\epsilon_0^2
    +
    2n\|y_{\bm{x}}\|_{\infty,0}^2\varepsilon^2
    \Big)
    +
    4nN^{2d}\varepsilon^2
    \big(
    1
    +
    \epsilon_0
    \big)^2
    +
    2L^2h^2.
\end{aligned}
\end{equation}

When $\bm x\in\Omega\setminus\Omega_{\varepsilon}$ where $\Omega_{\varepsilon}\coloneqq \bigcup_{\bm{k}}I_{\bm{k}, \varepsilon}$, since $\sum_{i=1}^{\bm{k}}|z_{\bm{k}}|\leq (1+N^d\epsilon)$, we have:
\begin{equation}
\begin{aligned}
    \E[\|\hat{\bm{y}}_{\bm{x}}\|^2]\leq n(1+N^d\varepsilon)^2\|\hat\psi_{\bm k}\|^2\leq n(1+N^d\varepsilon)^2(1+\epsilon_0)^2.
    \end{aligned}
\end{equation}
Therefore, we conclude that:
\begin{equation}
\begin{aligned}
    \int_{\Omega} W_2^2(f_{\bm{x}}, \hat{f}_{\bm{x}})
    \,\nu(\d\bm{x}) &\leq
    \int_{\Omega_{\varepsilon}}
    W_2^2(f_{\bm{x}}, \hat{f}_{\bm{x}})
    \,\nu(\d\bm{x})
    + 2\nu(\Omega\setminus\Omega_{\varepsilon})
    \sup_{\bm{x}}
    \Big(
    \E[\|\hat{\bm{y}}_{\bm{x}}\|^2]
    +
    \E[\|{\bm{y}}_{\bm{x}}\|^2]
    \Big)\\
    &\leq
    \sum_{\bm{k}}
    \int_{\Omega_{\varepsilon,\bm{k}}}
    W_2^2(f_{\bm{x}}, \hat{f}_{\bm{x}})
    \,\nu(\d\bm{x})
    +
    2\nu(\Omega\setminus\Omega_{\varepsilon})
    \sup_{\bm{x}}
    \Big(
    \E[\|\hat{\bm{y}}_{\bm{x}}\|^2]
    +
    \E[\|{\bm{y}}_{\bm{x}}\|^2]
    \Big)\\
    &\leq
    4\Big(
    2(1+\varepsilon)^2n\epsilon_0^2
    +
    2n\|y_{\bm{x}}\|_{\infty,0}^2\varepsilon^2
    \Big)
    +
    4N^{2d}n\varepsilon^2(1+\epsilon_0)^2
    +
    2L^2h^2\\
    &\quad
    +
    2n\nu(\Omega\setminus\Omega_{\varepsilon})
    \Big(
    (1+N^d\varepsilon)^2(1+\epsilon_0)^2
    +
    1
    \Big).
\end{aligned}
\end{equation}
Since \(h\), \(\varepsilon\), and \(\epsilon_0\) are chosen independently and arbitrarily, and since \(\nu(\Omega\setminus\Omega_{\varepsilon})\to 0\) as \(\varepsilon\to 0\), the proof of Theorem~\ref{theorem2} is complete.

In particular, the number of RBFs required to approximate the indicator function \(I_{\bm{x}\in I_{\bm{k}}}\) depends on the parameter \(\varepsilon\). Furthermore, both the number of cells \(I_{\bm{k}}\) and the choice of \(\varepsilon\) needed to ensure that \(\nu(\Omega\setminus\Omega_{\varepsilon})\) remains sufficiently small depend on the underlying probability measure \(\nu(\cdot)\).

\section{Settings and hyperparameters of numerical experiments}
\label{training_details}
We list the hyperparameters and settings for each example in Table~\ref{tab:setting}.

{\scriptsize \begin{table}[!ht]
\centering
\caption{\footnotesize 
Hyperparameter settings and neural network configurations used in the numerical experiments for each example. The trainable parameters of the hierarchical RBF-KAN shown in Fig.~\ref{fig:rbfnn} include the weights \(w_{i,j,k}\) and biases \(b_k^j\) associated with the linear layers, as well as the RBF coefficients \(\alpha_{i,j}^k\), scale parameters \(\beta_i^k\), and RBF centers \(c_{i,j}^k\). 
For the hierarchical RBF-SKAN shown in Fig.~\ref{fig:rbfsnn}, the trainable parameters of the first deterministic block are identical to those of the hierarchical RBF-KAN. The trainable parameters of the second stochastic block include the weights (\textit{e.g.} \(w_{i,j,k,r,s}^3\)), biases (\textit{e.g.} \(b_{r,s}^3\)), RBF coefficients \(\alpha_{i,j}^k\), lower bounds of the scale parameters \(a_{i,j,s}^k\), ranges of the scale parameters \(b_{i,j,s}^k-a_{i,j,s}^k\), and RBF centers \(c_{i,j}^k\).} 
{\scriptsize\begin{tabular}{lllll}
\toprule
 & Example~\ref{example1} & Example~\ref{example2}  & Example~\ref{example3}\\
\midrule
gradient descent method & Adam & Adam  & Adam\\
learning rate & 0.002  & 0.001  & 0.002\\
number of epochs &5000  & 10000  & 2000 \\
Number of training samples & 2000 & 30 & 2000\\
Number of testing sample points $N_1$ & 2000 & 30 & 1000\\
number of blocks & 3  &2 & 2 \\
$\ell$ (the number of RBFs to approximate each $\phi_{q, p}$)  & 4  &2 & 4  \\
$n_0$ (the number of outputs of each intermediate block)  & 8  &16 & 16  \\
initialization for NN parameters  & $\mathcal{N}(0, 0.1^2)$ & $\mathcal{N}(0, 0.1^2)$  & $\mathcal{N}(0, 0.01^2)$ & \\
\bottomrule
\end{tabular}}
\label{tab:setting}
\end{table}}

\newpage
\section{Sensitivity analysis on the hyperparameters in the hierarchical RBF-KAN}
\label{appendixB}
Here, we investigate the sensitivity of the proposed hierarchical RBF-KAN to the two hyperparameters $\ell$ and $n_0$ for learning the multivariate function in Eq.~\eqref{example1_model} from Example~\ref{example1}. The hierarchical RBF-KAN employed in this experiment consists of three blocks, identical to the architecture used in Example~\ref{example1}. Except for $\ell$ and $n_0$, all other training settings remain the same as those listed in Table~\ref{tab:setting}. 
As shown in Table~\ref{example1_sensitivity:tab1}, choosing a sufficiently small value of $n_0\leq 4$, \textit{i.e.}, the number of neurons in each intermediate block, leads to inaccurate reconstruction of Eq.~\eqref{example1_model}. In contrast, increasing $\ell$ from $1$ to $4$, where $\ell$ denotes the number of RBFs used to approximate each $\phi_{q,p}$ and $\Phi_q$, improves the reconstruction accuracy of Eq.~\eqref{example1_model}. 
However, Table~\ref{example1_sensitivity:tab2} shows that increasing either $n_0$ or $\ell$ also increases the computational runtime of the hierarchical RBF-KAN. Determining an optimal hyperparameter pair $(\ell, n_0)$ to balance computational cost and approximation accuracy remains an important direction for future research, as the optimal configuration is likely to depend on the specific problem under consideration.

{\tiny \begin{table}[!b]
\centering
\caption{Errors on the testing set obtained using the proposed three-block hierarchical RBF-KAN equipped with the ResNet technique for different choices of the hyperparameters \(\ell\) and \(n_0\).}
\label{tab:three_blocks}
\tiny
\begin{tabular}{ccccccc}
\toprule
Dimension of $\bm{x}$  & 1 & 2 & 3 & 4 & 5 & 6 \\
\midrule
$\ell=4, n_0=1$ & 9.7636 & 5.3604 & 0.2122 & 0.0696 & 0.0282 & 0.0481
 \\
$\ell=4, n_0=2$ & 7.9136 & 0.2691 & 0.3136 & 0.0413 & 0.0492 & 0.0313
 \\
 $\ell=4, n_0=4$ & 0.2080 & 0.1160 & 0.0197 & 0.0371 & 0.0164 & 0.0428 \\
$\ell=1, n_0=8$ & 0.0242 & 0.0182 & 0.0289 & 0.0290 & 0.0237 & 0.0313 \\
$\ell=2, n_0=8$  & 0.0217 & 0.0186 & 0.0251 & 0.0368 & 0.0313 & 0.0317 \\
$\ell=4, n_0=8$ & 0.0076 & 0.0120 & 0.0199 & 0.0229 & 0.0166 & 0.0329 \\
$\ell=8, n_0=8$ & 0.0132 & 0.0123 & 0.0145 & 0.0301 & 0.0281 & 0.0359 \\
\bottomrule
\end{tabular}
\label{example1_sensitivity:tab1}
\end{table}}

{\tiny \begin{table}[!b]
\centering
\caption{Training runtime of the proposed three-block hierarchical RBF-KAN equipped with the ResNet technique for different choices of the hyperparameters \(\ell\) and \(n_0\).}
\label{tab:three_blocks}
\tiny
\begin{tabular}{ccccccc}
\toprule
Dimension of $\bm{x}$  & 1 & 2 & 3 & 4 & 5 & 6 \\
\midrule
$\ell=4, n_0=1$ & 79.2 & 111.9 & 67.0 & 53.3 & 64.2 & 59.6
 \\
$\ell=4, n_0=2$ & 70.3 & 71.3 & 60.3 & 56.5 & 65.3 & 72.9
 \\
$\ell=4, n_0=4$ & 77.5 & 79.9 & 79.1 & 90.4 & 127.5 & 129.0 \\
 $\ell=1, n_0=8$ & 153.3 & 160.4 & 208.1 & 179.3 & 226.2 & 326.4 \\
$\ell=2, n_0=8$  & 196.3 & 212.5 & 242.4 & 237.1 & 292.9 & 357.2 \\
$\ell=4, n_0=8$ & 437.6 & 394.3 & 504.4 & 655.9 & 571.8& 566.8 \\
$\ell=8, n_0=8$ & 709.9 & 811.4 & 857.4 & 1044.2 & 956.6 & 994.2
 \\
\bottomrule
\end{tabular}
\label{example1_sensitivity:tab2}
\end{table}}


\end{document}